\documentclass[11pt,letterpaper]{article}
\usepackage[margin=1in]{geometry}
\usepackage[ruled,linesnumbered]{algorithm2e}
\usepackage[utf8]{inputenc}
\usepackage{amsmath, nicefrac}
\usepackage{amsthm}
\usepackage{thmtools}
\usepackage{thm-restate}
\usepackage{amssymb,verbatim, indentfirst}
\usepackage{mathtools, bbm, xcolor, mathtools}

\newtheorem{theorem}{Theorem}[section]
\newtheorem{definition}[theorem]{Definition}

\newtheorem{claim}{Claim}[section]

\newtheorem{problem}{Problem}

\newtheorem{lemma}[theorem]{Lemma}
\newtheorem{corollary}[theorem]{Corollary}

\newtheorem{remark}[theorem]{Remark}

\usepackage{graphicx}
\usepackage{booktabs}
\usepackage{url}
\usepackage[title]{appendix}
\usepackage[utf8]{inputenc} % allow utf-8 input
\usepackage[hidelinks]{hyperref}       % hyperlinks
\usepackage{url}            % simple URL typesetting
\usepackage{booktabs}       % professional-quality tables
\usepackage{amsfonts}       % blackboard math symbols
\usepackage{nicefrac}       % compact symbols for 1/2, etc.
\usepackage{microtype}      % microtypography

\usepackage{amsmath}
\usepackage{bbm}
\usepackage{verbatim}
\usepackage[font=small]{caption}
\usepackage{amsthm}
\usepackage{subcaption}
\usepackage{xcolor}
\graphicspath{ {images/} }

\usepackage{subfiles}
\DeclareMathOperator{\E}{\mathbb{E}}

\DeclareMathOperator*{\argmin}{arg\,min}

\newcommand{\norm}[1]{\|#1\|}
\newcommand{\wh}{\widehat}
\newcommand{\R}{\mathbb{R}}
\newcommand{\eps}{\varepsilon}

\title{L1 Regression with Lewis Weights Subsampling}
\author{Aditya Parulekar\\\texttt{adityaup@cs.utexas.edu }\\UT Austin \and Advait Parulekar\\\texttt{advaitp@utexas.edu}\\UT Austin \and Eric Price\\\texttt{ecprice@cs.utexas.edu}\\UT Austin}

\begin{document}
\maketitle
%\documentclass[12pt]{article}
% \usepackage[margin=0.25in]{geometry}
%\usepackage[utf8]{inputenc}
%%\usepackage{amsmath, amssymb, amsthm, bbm, xcolor}
% \setlength\parindent{0pt}
%\begin{abstract}
%\end{abstract}

\begin{abstract}
  We consider the problem of finding an approximate solution to
  $\ell_1$ regression while only observing a small number of labels.
  Given an $n \times d$ unlabeled data matrix $X$, we must choose a
  small set of $m \ll n$ rows to observe the labels of, then output an
  estimate $\wh{\beta}$ whose error on the original problem is within
  a $1 + \eps$ factor of optimal.  We show that sampling from $X$
  according to its Lewis weights and outputting the empirical
  minimizer succeeds with probability $1-\delta$ for
  $m > O(\frac{1}{\eps^2} d \log \frac{d}{\eps \delta})$.  This is
  analogous to the performance of sampling according to leverage
  scores for $\ell_2$ regression, but with exponentially better
  dependence on $\delta$.  We also give a corresponding lower bound of
  $\Omega(\frac{d}{\eps^2} + (d + \frac{1}{\eps^2}) \log
  \frac{1}{\delta})$.

  % The standard $\ell_1$ regression solution $\beta^*$ minimizes
  % $\norm{X \beta^* - y}_1$, for a data matrix $X \in \R^{n \times d}$
  % and set of labels $y \in \R^n$.  In our problem we are given the
  % unlabeled data $X$ and choose a small set $S \subset [n]$ to observe
  % the labels $y_S$, from which we should output a
  % $(1 + \eps)$-approximate $\wh{\beta}$ that is .

\end{abstract}

\section{Introduction}
The standard linear regression problem is, given a data matrix
$X \in \R^{n \times d}$ and corresponding values $y \in \R^n$, to find
a vector $\beta \in \R^d$ minimizing $\norm{X \beta - y}_p$.  Least
squares regression ($p = 2$) is the most common, but least absolute
deviation regression ($p = 1$) is sometimes preferred for its
robustness to outliers and heavy-tailed noise.  In this paper we focus
on $\ell_1$ regression:

\begin{equation}
\label{eq:LAD_objective}
\beta^* = \argmin_{\beta\in \mathbb{R}^d}\Vert X\beta-y\Vert_1
\end{equation}

But what happens if the unlabeled data $X$ is cheap but the labels $y$
are expensive?  Can we choose a small subset of indices, only observe
the corresponding labels, and still recover a good estimate
$\wh{\beta}$ of the true solution?  We would like an algorithm that
works with probability $1-\delta$ for any input $(X, y)$; this
necessitates that our choice of indices be randomized, so the
adversary cannot concentrate the noise on them.  Formally we define the
problem as follows:

% In this work, we consider an active learning scenario in which we see
% the data matrix $X$ and must query an expensive oracle for the labels
% $y_i$ corresponding to row $i$ (we refer to this as ``querying'' or
% ``sampling'' row $i$). As such, we are tasked with finding an
% approximately optimal solution $\wh \beta$ to
% (\ref{eq:LAD_objective}) using a small number of queries $m\ll n$.

\begin{problem}[Active L1 regression]\label{matrix_problem}
  There is a known matrix $X\in \mathbb{R}^{n\times d}$ and a fixed unknown vector $y$. A learner interacts with the instance by querying rows indexed $\{i_k\}_{k\in [m]}$ adaptively, and is shown labels $\{y_{i_k}\}_{k\in [m]}$ corresponding to the rows queried. The learner must return $\wh{\beta}$ such that with probability $1-\delta$ over the learner's randomness,
  \begin{equation}\label{eq:matrix_objective}
    \Vert X\wh{\beta}-y\Vert_1 \le (1+\varepsilon)\min_{\beta}\Vert X\beta-y\Vert_1.
  \end{equation}
\end{problem}

Some rows of $X$ may be more important than others. For example, if
one row is orthogonal to all the others, we need to query it to have
any knowledge of the corresponding $y$; but if many rows are in the
same direction it should suffice to label a few of them to predict the
rest.

A natural approach to this problem is to attach some notion of
``importance'' $p_1, \dotsc, p_n$ to each row of $X$, then sample rows
proportional to $p_i$.  We can represent this as a
``sampling-and-reweighting'' sketch $S \in \R^{m \times n}$, where
each row is $\frac{1}{p_i} e_i$ with probability proportional to
$p_i$.  This reweighting is such that
$\E_S[\norm{Sv}_1] \propto \norm{v}_1$ for any vector $v$.  By querying
$m$ rows we can observe $Sy$, and so can output the empirical risk
minimizer (ERM)
\begin{align}
  \wh{\beta} := \argmin \norm{SX \beta - Sy}_1.\label{eq:ERM}
\end{align}
For fixed $\beta$,
$\E_S \norm{SX\beta - Sy}_1 \propto \norm{X \beta - y}_1$.  The hope
is that, if the $p_i$ are chosen carefully, the ERM $\wh{\beta}$ will
satisfy~\eqref{eq:matrix_objective} with relatively few samples.  Our
main result is that this is true if the $p_i$ are drawn according to
the $\ell_1$ Lewis weights:

\begin{theorem}[Informal]\label{thm:main_informal}
  Problem~\ref{matrix_problem} can be solved with
  $m = O(\frac{1}{\varepsilon^2}d\log \frac{d}{\varepsilon\delta})$
  queries.  For constant $\delta = \Theta(1)$,
  $m = O\left(\frac{1}{\varepsilon^2}d\log d\right)$ suffices.
\end{theorem}

Note that, while the model allows for adaptive queries, this algorithm
is nonadaptive.

% Our analysis relies on the fact that sub-sampling the data matrix in this manner results in a subspace-embedding of the column space of $X$, and we show that this is sufficient to also preserve regression error. 

We next show that our sample complexity is near-optimal by
demonstrating the following lower bound on the number of queries
needed by any algorithm to obtain an accurate estimate.
\begin{theorem}[Informal]\label{thm:lowerbound_informal}
  Any algorithm satisfying Problem \ref{matrix_problem} must query
  $\Omega(d\log\frac{1}{\delta}+\frac{d}{\varepsilon^2}+\frac{1}{\varepsilon^2}\log\frac{1}{\delta})$
  rows on some instances $(X, y)$.
\end{theorem}

For small $\delta$, the upper bound is the product of
$d, \frac{1}{\eps^2},$ and $\log (1/\delta)$ while the lower bound is
the product of each pair.

\subsection{Related Work}

% . 

\paragraph{If all the labels are known:}
\label{sec:labels_known} LAD regression cannot be solved in closed
form. It can be written as a linear program, but this is relatively
slow to solve.  One approach to speeding up LAD regression is
``sketch-and-solve,'' which replaces~\eqref{eq:LAD_objective}
with~\eqref{eq:ERM}, which has fewer constraints and so can be solved
faster.  The key idea here is to acquire regression guarantees by
ensuring that $S$ is a \textit{subspace embedding} for the column
space of $[X~y]$.

For a survey on techniques to do this, we direct the reader to
\cite{Woodruff_2014},\cite{mahoney2011randomized},
\cite{clarkson2005subgradient}. In \cite{Woodruff_2014}, the emphasis
is on \textit{oblivious} sketches -- distributions which do not
require knowledge of $[X~y]$. On the other hand,
\cite{mahoney2011randomized}, \cite{clarkson2005subgradient} discuss
sketches that depend on $[X~y]$. Most relevant to us
\cite{durfee2018ell1}, which shows that sampling-and-reweighting
matrices $S$ using Lewis weights of $[X~y]$ suffice; we give a simple
proof of this in Remark~\ref{remark:known_y}.  The problem is that
\emph{figuring out which labels are important} involves looking at all
the labels.

% Because we do not have access to all of the labels $y$, and we would like an algorithm that only uses $X$, not $[X~y]$ to decide which subset of labels to use.  

\paragraph{Active $\ell_2$ regression:}
Here we return to our setting, where only a subset of the labels is
available to us. A number of works have studied this problem,
including~\cite{drineas2006sampling,derezinski2017unbiased,derezinski2021determinantal}.
The $\ell_2$ version of the problem was solved optimally in
\cite{chen2019active}, where an algorithm was given using
$O(\frac{d}{\varepsilon})$ queries to find $\wh\beta$ satisfying
$\E\left[\Vert X\wh\beta-y\Vert_2^2\right] \le (1+\varepsilon)\Vert
X\beta^*-y\Vert_2^2$.  Independent, identical sampling using leverage
scores achieves the same guarantee using
$O(d\log d + \frac{d}{\varepsilon})$ queries. Note that these results
for $\ell_2$ ERM only work in expectation, while our results hold with
high probability.  One can get high probability bounds in the $\ell_2$
setting by taking the median of $O(\log 1/\delta)$ repetitions, but
the ERM itself does not succeed with high probability.

\paragraph{Subspace embedding for $\ell_1$ norms:}
Subspace embeddings for the $\ell_1$ norm have been studied in a long
line of work including \cite{talagrand1990embedding},
\cite{talagrand1995embedding}, \cite{ledoux1989},
\cite{dasgupta2009sampling}, and \cite{cohen2014ellp}, the most recent
of which describes an iterative algorithm to approximate \textit{Lewis
  weights}, which are the analogue of leverage scores for importance
sampling preserving $\ell_1$ norms. The \cite{cohen2014ellp} result
shows that, for the same
$m = O(\frac{1}{\eps^2} d \log \frac{d}{\eps \delta})$ sample
complexity as given in Theorem~\ref{thm:main}, a sampler sketch $S$
based on the Lewis weights of $X$ will have
$\norm{SX\beta}_1 \approx_\eps \norm{X\beta}_1$ for all
$\beta \in \R^d$.

% Our main results states that sampling in an identical
% way to \cite{cohen2014ellp} results in good guarantees even for
% regression, even though the sampling procedure does not see $y$.

\paragraph{Our approach.}
At a very high level the goal of this paper is to replace the $\ell_2$
leverage score analysis of the~\cite{chen2019active} active regression
paper with the $\ell_1$ Lewis weight analysis in
the~\cite{cohen2014ellp} subspace embedding paper.  However, the
differences between $\ell_1$ and $\ell_2$ are significant enough that
very little of the~\cite{chen2019active} proof approach remains.

Per~\cite{cohen2014ellp}, the Lewis weight sampling-and-embedding
matrix $S$ preserves $\norm{X \beta}_1$ for all $\beta$.  The problem
is that it \emph{doesn't} preserve $\norm{X \beta - y}_1$: if $y$ has
outliers, we have no idea where they are to sample them.  In the
$\ell_2$ setting, this difficulty is addressed using the closed-form
solution $\beta^* = X^{\dagger}y$.  Then if $S$ is a subspace
embedding it will preserve $\norm{X \beta - X \beta^*}$, so it
suffices to bound the expectation of $\norm{S(X \beta^* - y)}_2^2$.
In the $\ell_1$ setting, not only is $\beta^*$ not expressible in
closed form, but there can be many equally valid minimizers $\beta^*$
that are far from each other.  In Appendix~\ref{app:constant-factor}
we show how this approach extends to the $\ell_1$ setting to give a
simple proof of Theorem~\ref{thm:main_informal} for a constant factor
approximation (i.e., $\eps = O(1)$); but the existence of multiple
$\beta^*$ makes $\eps < 1$ seem unobtainable by this approach.

Instead, we massage the~\cite{chen2019active} subspace embedding proof
into the appropriate form, as we discuss in
Section~\ref{sec:proofoverview}.  While $S$ doesn't preserve the total
error $\norm{X \beta - y}_1$, it does preserve \emph{relative} error
$\norm{X \beta - y}_1 - \norm{X \beta^* - y}_1$; the effect of
outliers is canceled out, so that this concentrates similarly well to
$\norm{X \beta - X \beta^*}_1$.  This approach would not work for
$\ell_2$: the effect of outliers does not entirely cancel out there,
since the square loss has unbounded influence.

\paragraph{Concurrent work:}
A very similar set of results appears concurrently and independently
in~\cite{chen2021query}.  Their main result is identical to ours, with
a similar proof.  They also extend the result to $1 < p < 2$, but with
a significantly weaker $m = \widetilde{O}(d^2/\eps^2)$ bound.  They do
not have the $\Omega(d \log \frac{1}{\delta})$ lower bound.

\section{Preliminaries: Subspace Embeddings and Importance Sampling}\label{sec:prelim}
A key idea used in our analysis is that of a $\ell_1$ subspace embedding, which is a linear sketch of a matrix that preserves $\ell_1$ norms within the column space of a matrix:
\begin{definition}[Subspace Embeddings]
A subspace embedding for the column space of the matrix $X \in \mathbb{R}^{n\times d}$ is a matrix $S$ such that for all $\beta \in \mathbb{R}^d$, 
\[\Vert SX\beta\Vert = (1\pm \varepsilon) \Vert X\beta\Vert\]
\end{definition}

\begin{remark}\label{remark:known_y}
Consider the simpler setting in which we had access to all of $y$, but we still want to subsample rows to improve computational complexity. We can view the regression loss $\Vert X\beta-y\Vert_1$ as the $\ell_1$ norm of the point $[X~y]\begin{bmatrix}\beta\\ -1\end{bmatrix}$ in the column space of $[X~y]$. Indeed, suppose $\beta^*=\argmin \Vert X\beta-y\Vert_1$ as before and let $\wh{\beta} = \argmin \Vert SX\beta-Sy\Vert_1$. Then, $\wh\beta$ solves problem \ref{matrix_problem} because, for $\varepsilon<\frac{1}{3}$,
\[
  \Vert X\wh\beta-y\Vert_1\le \frac{1}{1-\varepsilon} \Vert SX\wh \beta-Sy\Vert_1 \le \frac{1}{1-\varepsilon} \Vert SX\beta^*-Sy\Vert_1 \le \frac{1+\varepsilon}{1-\varepsilon} \Vert X\beta^*-y\Vert_1\le (1+4\varepsilon)\Vert X\beta^*-y\Vert_1.
\]
\end{remark}
One way to construct a subspace embedding is by sampling rows and rescaling them appropriately:
\begin{definition}[Sampling and Reweighting with $\{p_i\}_{i=1}^n$]\label{defn:sampling_reweighting}
For any sequence $\{p_i\}_{i = 1}^n$, let $N = \sum_i p_i$. Then, the sampling-and-reweighting distribution $\mathcal{S}\left(\{p_i\}_{i = 1}^n\right)$ over the set of matrices $S\in \mathbb{R}^{N\times n}$ is such that each row of $S$ is independently the $i$th standard basis vector with probability $\frac{p_i}{N}$, scaled by $\frac{1}{p_i}$. For any $k\in [N]$, let $i_k$ denote the index such that $S_{k,i_k} = \frac{1}{p_{i_k}}$.
\end{definition}
% We sub-sample the rows of $X$ independently and identically in proportion to their Lewis weights, and acquire labels only for the rows that were sampled. We then solve (\ref{eq:LAD_objective}) on this smaller data, and provide guarantees against the estimate we would have made had we seen all of the labels and then solved (\ref{eq:LAD_objective}). 
% This closely resembles the analysis of \cite{cohen2014ellp} in which this is done to obtain a subspace embedding. The existing result for $\ell_1$ norm subspace embeddings, as described in \cite{cohen2014ellp} gives a bound on the number of rows needed to accurately represent a column space.

When working in $\ell_2$, there is a natural choice for re-weighting: the leverage scores of the rows \cite{Woodruff_2014}. 
\begin{definition}[Leverage Scores]\label{defn:leverage_scores}
The leverage score of the $i$th row of a matrix $X$, $l_i(X)$ is defined as $x_i^\top (X^\top X)^{-1}x_i$.
\end{definition}
For $\ell_1$ subspace embeddings, the analogous weights are the $\ell_1$ Lewis weights, defined implicitly as the unique weights $\{w_i(X)\}_{i=1}^n$ that satisfy $w_i(X) = l_i(WX)$ where $W$ is a diagonal matrix with $i$th diagonal entry $\frac{1}{\sqrt{w_i(X)}}$. We will drop the explicit dependence on $X$ whenever it is clear from context.
\begin{definition}[Lewis Weights]\label{defn:lewis_weights}
The $\ell_1$ Lewis weights of a matrix $X$ are the unique weights $\{w_i\}_{i=1}^n$ that satisfy $w_i^2 = x_i^\top (\sum_{j=1}^n \frac{1}{w_j}x_jx_j^\top )^{-1}x_i$ for all $i$.
\end{definition}
Lewis weights are defined in general for general $\ell_p$ norms, but we will only need the $\ell_1$ Lewis weights. For basic properties of Lewis weights, we direct the reader to \cite{cohen2014ellp}. Using these definitions, we now state the main consequence of using Lewis weights. This result comes from a line of work on embeddings from subspaces of $L_1[0, 1]$ to $\ell_1^m$ such as \cite{talagrand1990embedding}, but is reproduced here similar to how it is presented in \cite{cohen2014ellp}. 
\begin{theorem}[\cite{cohen2014ellp} Theorem 2.3]\label{thm:ell_1_embedding}
    Sampling at least $O(\frac{d\log d}{\varepsilon^2})$ rows according to the $\ell_1$ Lewis weights $\{w_i\}_{i=1}^n$ of a matrix $X\in \mathbb{R}^{n\times d}$ results in a subspace embedding for $X$ with at least some constant probability.
    %preserves norms in the column space of $X$ to within a multiplicative factor $\varepsilon$. Concretely;
    %$$\Vert SX\beta\Vert_1 = (1\pm \varepsilon)\Vert X\beta \Vert_1$$ with constant probability over $S\sim \mathcal{S}$ where $\mathcal{S}$ is a sampling and reweighting distribution with $\{l_i\}_{i=1}^n$ for all $\beta\in\mathbb{R}^d$. 
    If at least $O(\frac{d\log \frac{d}{\varepsilon\delta}}{\varepsilon^2})$ rows are sampled, then we have a subspace embedding with probability at least $1-\delta$.
\end{theorem}

\subsection{Properties of Lewis Weights}
We will need some properties of Lewis weights, particularly of how they change when the matrix $X$ is modified. 
\begin{lemma}[\cite{cohen2014ellp} Lemma 5.5]\label{lemma:lewrows}
The $\ell_1$ Lewis weights of a matrix do not increase when rows are added.
\end{lemma}
\begin{restatable}{lemma}{lewbound}\label{lemma:lewbound}
Let $X \in \mathbb{R}^{n\times d}$, and let $X' \in \mathbb{R}^{kn\times d}$ be $X$ stacked on itself $k$ times, with each row scaled down by $k$. Then, each of the Lewis weights is reduced by a factor of $k$.
\end{restatable}
% This is proved in the appendix.
\section{Proof Overview}\label{sec:proofoverview}
% We state our main result below.
\begin{theorem}\label{thm:main}
  Let $X\in\mathbb{R}^{n\times d}$ have $\ell_1$ Lewis weights
  $\{w_i\}_{i\in [n]}$, and let $0 < \eps, \delta < 1$. Then, for any
  $N$ that is at least
  $O\left(\frac{d}{\varepsilon^2}
    \log\frac{d}{\varepsilon\delta}\right)$, there is a
  sampling-and-reweighting distribution $\mathcal{S}(\{p_i\}_{i=1}^n)$
  satisfying $\sum_i p_i = N$ such that for all $y$, if
  $S\sim \mathcal{S}(\{p_i\}_{i=1}^n)$ and
  $\wh\beta = \arg\min\Vert SX\beta - Sy\Vert_1$, we have
\[\Vert X\wh\beta - y\Vert_1 \le (1 + \varepsilon) \min_\beta\Vert X\beta - y\Vert_1\]
with probability $1-\delta$. If $\delta=O(1)$ is some constant, then $N$ at least $O\left(\frac{1}{\varepsilon^2}d\log d\right)$ rows suffice.
\end{theorem}

\paragraph{Regression guarantees from column-space embeddings. } 

As noted in Remark \ref{remark:known_y}, it would suffice to show that $\Vert SX\beta - Sy\Vert_1 \approx \Vert X\beta - y\Vert_1$ for all $\beta$. The problem is that this is impossible without knowing $y$: if one random entry of $y$ is very large, we would need to sample it to estimate $\Vert X \beta - y\Vert_1$ accurately. 
%A weaker analogous claim, that $\Vert SX\beta^*-Sy\Vert_2\approx \Vert X\beta^*-y\Vert_2$ is used in \cite{chen2019active} for the $\ell_2$ results, but even that is not true 
%
%[Maybe the bit about L2 estimating $beta^*$ here? not sure.]
%
However, we don't actually need to estimate $\Vert X \beta - y\Vert_1$; we just need to be able to distinguish values of $\beta$ for which $\Vert X\beta-y\Vert_1$ is far from $\Vert X\beta^*-y\Vert_1$ from values for which it is close. That is, it would suffice to accurately
\begin{equation}\label{estimation_objective}
\text{estimate} \hspace{0.5cm} \Vert X \wh\beta - y\Vert_1 - \Vert X \beta^* - y\Vert_1 \hspace{0.5cm} \text{with} \hspace{0.5cm} \Vert SX \wh\beta - Sy\Vert_1 - \Vert SX \beta^* - Sy\Vert_1       
\end{equation}
for every possible $\beta$.  In the above example where $y$ has a single large outlier coordinate, sampling this coordinate or not will dramatically affect \textit{both} terms, but will not affect the difference very much. As such, our key lemma, Lemma \ref{thm:mainlemma}, states that $\ell_1$ Lewis weight sampling achieves (\ref{estimation_objective}) with high probability. In particular, using at least $m \ge O(\frac{d}{\varepsilon^2}\log \frac{d}{\varepsilon\delta})$ rows we have
\begin{equation}\label{eq:deviation}(\Vert SX\beta^* - Sy\Vert_1-\Vert SX\beta - Sy\Vert_1)-(\Vert X\beta^* - y\Vert_1-\Vert X\beta - y\Vert_1) <\varepsilon \Vert X(\beta^*-\beta)\Vert_1
\end{equation}
for all $\beta$ with probability at least $1-\delta$. We do this by adapting the argument of \cite{cohen2014ellp} which shows that $S$ is a column-space embedding with high probability. We have summarized this argument below.
% We show that sampling-and-reweighting using Lewis weights approximately preserves (\ref{eq:residual_error}) using an argument adapted from \cite{cohen2014ellp}, which shows that such a matrix with at least $O(\frac{d}{\varepsilon^2}\log \frac{d}{\varepsilon\delta})$ rows is a subspace embedding of the column space of $X$ with high probability.

% Here, while we cannot use such orthogonality, we show a similar inequality in Lemma \ref{thm:mainlemma} which demonstrates that with high probability
% \begin{equation}\label{eq:stronger_triangle_inequality}
% \Vert X\wh\beta-y\Vert_1 - \Vert X\beta^*-y\Vert_1\le \varepsilon \Vert X\wh\beta - X\beta^*\Vert_1.
% \end{equation}
\paragraph{Column-space embedding using Lewis weights (\cite{cohen2014ellp}).}
An important result in \cite{cohen2014ellp}, which directly implies the high probability subspace embedding, and which will be useful to us later is the following moment bound on deviations of $\Vert SX\beta\Vert_1$.
\begin{lemma}[\cite{cohen2014ellp} Lemma 7.4]\label{lemma:moment}
If $N$ is at least $O\left(\frac{d}{\varepsilon^2}\log\frac{d}{\varepsilon\delta}\right)$, and $S\in \mathbb{R}^{N\times n}$ is drawn from the sampling-and-reweighting distribution $\mathcal{S}(\{p_i\}_{i=1}^N)$ with $\sum_i p_i = N$ and $\{p_i\}_{i=1}^n$ proportional to Lewis weights $\{w_i\}_{i=1}^n$, then 
\begin{align*}
    \mathop\mathbb{E}_S \left[\left(\max_{\Vert X\beta\Vert_1 = 1} \left|\Vert SX\beta\Vert_1 - \Vert X\beta\Vert_1\right|\right)^l\right] \le \varepsilon^l\delta
\end{align*}
\end{lemma}
The proof follows from this chain of inequalities:
\begin{align*}
    \mathop\mathbb{E}_S\left[\left(\max_{\Vert X\beta\Vert_1 = 1}\Vert SX\beta\Vert_1 - \Vert X\beta\Vert_1\right)^l\right] &\stackrel{(A)}{\le} 2^l\mathop\mathbb{E}_{\sigma, S}\left[\left(\max_{\Vert X\beta\Vert_1 = 1}\left|\sum_{k} \sigma_k \frac{|x_{i_k}^T\beta|}{p_{i_k}}\right|\right)^l\right] \\
    & \stackrel{(B)}{\le} 2^l \mathop\mathbb{E}_{\sigma, S}\left[\left(\max_{\Vert X\beta\Vert_1 = 1}\sum_{k} \sigma_k \frac{x_{i_k}^T\beta}{p_{i_k}}\right)^l\right] \\
    & \stackrel{(C)}{\le} \varepsilon^l\delta
\end{align*}
where the $\sigma_k$ are independent Rademacher variables, which are $\pm 1$ with probability $1/2$ each, and $p_{i_k}$ is proportional to the $\ell_1$ Lewis weight of row $i_k$. $(A)$ follows by symmetrizing the objective $F\coloneqq  \max_{\Vert X\beta\Vert_1=1} \Vert SX\beta\Vert_1-\Vert X\beta\Vert_1$. $(B)$ follows from a contraction lemma. $(C)$ is shown by constructing a related matrix with bounded Lewis weights and applying Lemma \ref{lemma:talhigh} from \cite{talagrand1990embedding} reproduced below. 
\begin{lemma}
There exists constant $C$ such that for any $X\in\mathbb{R}^{n\times d}$ with all $\ell_1$ Lewis weights less than $C \frac{\varepsilon^2}{\log\left(\frac{n}{\delta}\right)}$ and $l = \log(2n/\delta)$, we have
\begin{align}\mathbb{E}_{\sigma}\left[\left(\max_{\Vert X\beta\Vert_1= 1}\left|\sum_{i = 1}^n\sigma_ix_{i}^\top \beta\right|\right)^l\right]\le \frac{\varepsilon^l\delta}{2}\end{align}
\end{lemma}
\paragraph{Regression Guarantees using Lewis weight sampling.} In this work, we show the following chain of inequalities.
\begin{align}
    &\mathop\mathbb{E}_S\left[\left(\max_{\Vert X\beta^* - X\beta \Vert = 1}\left|\left(\Vert SX\beta^* - Sy\Vert_1 - \Vert SX\beta - Sy\Vert_1\right) - \left(\Vert X\beta^* - y\Vert_1 - \Vert X\beta - y\Vert_1\right)\right|\right)^l\right]\nonumber\\
    &\hspace{2cm}\stackrel{(A)}{\le}
    2^l\mathop\mathbb{E}_{S, \sigma}\left[\left(\max_{\Vert X\beta^* - X\beta \Vert = 1}\left| \sum_{k} \sigma_k\left(\frac{|x_{i_k}^\top \beta^* - y_{i_k}|}{p_{i_k}} -  \frac{|x_{i_k}^\top \beta - y_{i_k}|}{p_{i_k}}\right)\right|\right)^l\right]\nonumber\\
    &\hspace{2cm}\stackrel{(B)}{\le }
    2^{2l+1}\mathop\mathbb{E}_{S, \sigma}\left[\left(\max_{\Vert X(\beta^* -\beta) \Vert_1 = 1}\left| \sum_{k} \sigma_{i_k}\frac{x_{i_k}^\top}{p_{i_k}}(\beta^*-\beta)\right|\right)^l\right]\label{eq:SX_moment}\\
    &\hspace{2cm}\stackrel{(C)}{\le }\varepsilon^l \delta\nonumber
\end{align}

Here, for $(A)$, we symmetrize the left hand side of (\ref{eq:deviation}) in Lemma \ref{lemma:prelim_symmetrize}. For $(B)$, we apply a different contraction lemma, Lemma \ref{thm:talagrand}, that allows us to remove $y$ from our expression, and then end up with the same moment bound for $(C)$. Step $(C)$ is essentially an application of Lemma \ref{lemma:talhigh} to $SX$, however, because we cannot immediately bound the Lewis weights of $SX$ to confirm the constraints of the Lemma, we instead construct another matrix $X''$ which does not significantly alter the right hand side of inequality (\ref{eq:SX_moment}) while having bounded Lewis weights. This is done in Lemmas \ref{lemma:b1mod} and \ref{lemma:high}. 
\subsection{Lower Bounds}\label{sec:lower_bound}
We will show that any algorithm must see $\Omega(d\log\frac{1}{\delta}+\frac{1}{\varepsilon^2}\log\frac{1}{\delta}+\frac{d}{\varepsilon^2})$ labels to return $\wh \beta$ satisfying $\Vert X\wh\beta-y\Vert_1 \le (1+\varepsilon)\Vert X\beta^* - y\Vert_1$ with probability greater than $1-\delta$.

For the lower bound proof it is convenient to consider a \emph{distributional} version of the problem:
\begin{problem}[Distributional active L1 regression]\label{distribution_problem}
    There is an unknown joint distribution $P$ over a finite set $\mathcal{X}\times\mathcal{Y}\subset \mathbb{R}^d\times\mathbb{R}$, with $|\mathcal{Y}| = 2$. The learner is allowed to adaptively observe $N$ i.i.d. samples from $P(\cdot|X = x)$ for the learner's choice of $N$ values $x\in \mathcal{X}$. The learner must return $\wh\beta$ satisfying
    \begin{equation}\label{eq:distribution_objective}\E_{(X, Y)\sim P}\left[|X^\top\wh\beta-Y| \right]\le (1+\varepsilon)\inf_{\beta}\E_{(X, Y)\sim P}\left[|X^\top\beta-Y| \right].
    \end{equation}
    with probability at least $1-\delta$.
\end{problem}
We begin with a lemma that shows that solving the original, Problem~\ref{matrix_problem}, for some $n$ polynomial in the parameters $d,\varepsilon, \delta$ is harder than solving the distributional version, Problem~\ref{distribution_problem}. 
\begin{restatable}{lemma}{distr}\label{distribution_to_matrix}
A randomized algorithm that solves Problem \ref{matrix_problem} for $n = \frac{2}{\varepsilon^2}\left(\log \frac{2}{\delta} + d\log \frac{3d}{\varepsilon}\right)$ with accuracy $\varepsilon$ and failure probability $\delta$ can be used to solve any instance of Problem \ref{distribution_problem}, where $\mathcal{X}, \mathcal{Y}, $ in the unit $\ell_\infty$ ball, with accuracy $6\varepsilon$ and failure probability $2\delta$, for small $\varepsilon$.
\end{restatable}

We then prove lower bounds on the accuracy for any algorithm on Problem \ref{distribution_problem}.

In all our lower bounds, $x$ is a uniform $e_i$, and $y \in \{0, 1\}$.
For $\Omega(\frac{d}{\eps^2})$, we set $P(y | x = e_i)$ to
$\frac{1}{2} \pm \eps$ uniformly at random independently for each $i$;
getting an $\eps$-approximate solution requires getting most of the
biases correct, which requires $\frac{1}{\eps^2}$ samples from most of
the coordinates $e_i$.  The
$\Omega(\frac{1}{\eps^2} \log \frac{1}{\delta})$ instance sets
$P(y | x = e_i)$ to $\frac{1}{2} \pm \eps$ with the same bias for each
$i$; solving this is essentially distinguishing a $\eps$ biased coin
from a $-\eps$-biased coin.  Finally, for
$\Omega(d \log \frac{1}{\delta})$ we set $P(y | x = e_i) = 0$ except
for a random hidden $i^*$ with $P(y \mid x = e_{i^*}) = \frac{3}{4}$.
Solving this instance requires finding $i^*$, but there's a $\delta$
chance the first $d \log \frac{1}{\delta}$ queries are all zero.

\begin{theorem}\label{lower_bound}
For any $d\ge 2$, $\epsilon < \frac{1}{10}$, $\delta < \frac{1}{4}$, there exist sets $\mathcal{X}\in \mathbb{R}^d, \mathcal{Y}\in \mathbb{R}$ of inputs and labels, and a distribution $P$ on $\mathcal{X}\times \mathcal{Y}$ such that any algorithm which solves Problem \ref{distribution_problem}, with $\varepsilon = 1$, requires at least $m=\Omega(\frac{d}{\epsilon^2} + \frac{1}{\epsilon^2}\log\frac{1}{\delta} + d\log \frac{1}{\delta})$ samples.
\end{theorem}

\section{Proof of Theorem \ref{thm:main}}
\begin{lemma}\label{thm:mainlemma}
Let $X\in\mathbb{R}^{n\times d}$ have $\ell_1$ Lewis weights $\{w_i\}_{i\in [n]}$. Then, for any $N$ that is at least $O\left(\frac{d}{\varepsilon^2} \log\frac{d}{\varepsilon\delta}\right)$, there is a sampling-and-reweighting distribution $\mathcal{S}(\{p_i\}_{i=1}^n)$ satisfying $\sum_i p_i = N$ such that for all $y$, if $S\sim \mathcal{S}(\{p_i\}_{i=1}^n)$ and $\beta^* = \arg\min\Vert X\beta - y\Vert_1$, we have for all $\beta$
\begin{align}\left(\Vert SX\beta^* - Sy\Vert_1 - \Vert SX\beta - Sy\Vert_1\right) - \left(\Vert X\beta^* - y\Vert_1 - \Vert X\beta - y\Vert_1\right) \le \varepsilon\cdot \Vert X\beta^* - X\beta\Vert _1\end{align} 
with probability at least $1-\delta$. Further, for constant $\delta$, $m= O(d\log d/\varepsilon^2)$ rows suffice. 
\end{lemma}

This lemma is proved for constant and high probability bounds in Section~\ref{sec:lemma_proof}. Given this, we can prove the main theorem. 

\begin{proof}[Proof of Theorem \ref{thm:main}]
Applying Lemma \ref{thm:mainlemma} to $\wh\beta \coloneqq  \arg\min \Vert SX\beta - Sy\Vert_1$, we get
\begin{align*}
\left(\Vert SX\beta^* - Sy\Vert_1 - \Vert SX\wh\beta -Sy\Vert_1\right) \le \left(\Vert X\beta^* - y\Vert_1 - \Vert X\wh\beta - y\Vert_1\right) +\varepsilon\cdot \Vert X\beta^* - X\wh\beta\Vert _1
\end{align*}
Since $\wh\beta$ is the minimizer of $\Vert SX\beta - Sy\Vert_1$, the left side is non-negative. So, 
\begin{align*}
\Vert X\wh\beta - y\Vert_1 &\le \Vert X\beta^* - y\Vert_1 + \varepsilon\cdot\Vert X\beta^* - X\wh\beta\Vert _1\\
&\le \Vert X\beta^* - y\Vert_1 + \varepsilon\cdot(\Vert X\beta^* - y\Vert_1 + \Vert X\wh\beta - y\Vert_1)
\end{align*}
Rearranging, and assuming $\varepsilon < 1/2$,
\begin{align*}
\Vert X\wh\beta - y\Vert_1
	&\le \frac{1+\varepsilon}{1-\varepsilon}\Vert X\beta^* - y\Vert_1 \\
	&\le (1+4\varepsilon) \Vert X\beta^* - y\Vert_1
\end{align*}
Using $\varepsilon' = \varepsilon/4$ proves the theorem. 
\end{proof}
\subsection{Proof of Lemma \ref{thm:mainlemma}}\label{sec:lemma_proof}
This argument is similar to that in Appendix B of \cite{cohen2014ellp}. In order to prove Lemma \ref{thm:mainlemma}, by Markov's inequality, it is sufficient to show that for some $l$,
\[M \coloneqq  \mathop\mathbb{E}_S\left[\left(\max_{\Vert X\beta^* - X\beta \Vert = 1}\left|\left(\Vert SX\beta^* - Sy\Vert_1 - \Vert SX\beta - Sy\Vert_1\right) - \left(\Vert X\beta^* - y\Vert_1 - \Vert X\beta - y\Vert_1\right)\right|\right)^l\right]\le \varepsilon^l\delta\]
To show this, we will symmetrize, then use a contraction lemma to cancel the $y$ terms. Then, with all the terms being within the column space of $SX$, we use the fact that $S$ is a subspace embedding with high probability. We present two different bounds, one used for the constant probability and one for the high probability cases, but the following intermediate bound is the same for the two: 
\begin{restatable}{lemma}{prelimsymmetrize}\label{lemma:prelim_symmetrize}
Given a matrix $X\in \mathbb{R}^{n\times d}$, let $\mathcal{S}(\{p_i\}_{i\in [n]})$ be any sampling-and-reweighting disribution, and let $i_k$ be the row-indices chosen by this sampling matrix such that $S_{k,i_k} = \frac{1}{p_{i_k}}$. Let $\sigma_k$ be independent Rademacher variables that are $\pm 1$ each with probability $0.5$. Then, 
\begin{align}M \le  2^l\mathop\mathbb{E}_{S, \sigma}\left[\left(\max_{\Vert X\beta^* - X\beta \Vert = 1}\left| \sum_{k} \sigma_k\left(\frac{|x_{i_k}^\top \beta^* - y_{i_k}|}{p_{i_k}} -  \frac{|x_{i_k}^\top \beta - y_{i_k}|}{p_{i_k}}\right)\right|\right)^l\right]\label{eq:3}\end{align}
% \begin{align}M \le 2^{2l+1}\mathop\mathbb{E}_{S, \sigma}\left[\left(\max_{\Vert X\beta \Vert = 1}\left| \sum_{k} \sigma_k\frac{x_{i_k}^\top \beta}{p_{i_k}}\right|\right)^l\right\end{align}
\end{restatable}
This is essentially standard symmetrization; the proof is in Appendix~\ref{app:proofs}.  To simplify the expression and eliminate the terms involving the labels, we then use a theorem from \cite{ledoux1989}:
\begin{lemma}[\cite{ledoux1989} Theorem 5]\label{thm:talagrand}
Let $\Phi : \mathbb{R}^+\rightarrow\mathbb{R}^+$ be convex and increasing, and let $\phi_k:\mathbb{R}\rightarrow\mathbb{R}$ be contractions such that $\phi_k(0) = 0$ for all $k$. Let $\mathcal{F}$ be a class of functions on $\{1, 2, 3\dots, n\}$, and $\Vert g(f) \Vert_\mathcal{F} = \sup_{f\in\mathcal{F}} |g(f)|$. Then,
\[\mathbb{E}_\sigma\left[ \Phi\left( \frac{1}{2}\left\Vert \sum_k \sigma_k\phi_k(f(k))\right\Vert _\mathcal{F}\right)\right] \le \frac{3}{2}\mathbb{E}_\sigma\left[\Phi\left(\left\Vert\sum_k  \sigma_k f(k)\right\Vert_\mathcal{F}\right)\right]\]
\end{lemma}

\begin{lemma}\label{thm:triangle}
    For any $y\in \mathbb{R}^n$, we have 
    \begin{align}\mathop\mathbb{E}_{S, \sigma}&\left[\left(\max_{\Vert X\beta^* - X\beta \Vert = 1}\left| \sum_{k} \sigma_k\left(\frac{|x_{i_k}^\top \beta^* - y_{i_k}|}{p_{i_k}} -  \frac{|x_{i_k}^\top \beta - y_{i_k}|}{p_{i_k}}\right)\right|\right)^l\right]\nonumber\\
    &\le 2^{l+1}\mathop\mathbb{E}_{S, \sigma}\left[\left(\max_{\Vert X\beta^* - X\beta \Vert_1 = 1}\left| \sum_{k} \sigma_k\left(\frac{x_{i_k}^\top \beta^* - x_{i_k}^\top \beta}{p_{i_k}}\right)\right|\right)^l\right]\label{eq:triangle}
    \end{align}
\end{lemma}
\begin{proof}
We take $\Phi(x) = x^l$, which is convex and increasing for $l > 1$, let $\mathcal{F}$ be the set of functions $f_\beta$ where $f_\beta(k) = \frac{x_{i_k}^\top \beta^* - x_{i_k}^\top \beta}{p_{i_k}}$ and $\beta$ satisfies $\Vert X\beta^* - X\beta\Vert_1 = 1$, and let $\phi_k$ be defined as 
\[\phi_k(z) = \frac{|x_{i_k}^\top \beta^* - y_{i_k}|}{p_{i_k}} - \frac{|x_{i_k}^\top \beta^* - zp_{i_k} - y_{i_k}|}{p_{i_k}}.\]
This satisfies 
\[\phi_k(f_\beta(k)) = \phi_k\left(\frac{x_{i_k}^\top \beta^* - x_{i_k}^\top \beta}{p_{i_k}}\right) = \frac{|x_{i_k}^\top \beta^* - y_{i_k}|}{p_{i_k}} - \frac{|x_{i_k}^\top \beta - y_{i_k}|}{p_{i_k}}.\]
This is a contraction, since
\begin{align*}
|\phi_k(z_1) - \phi_k(z_2)| &= \left|\frac{|x_{i_k}^\top \beta^* - z_2p_{i_k} - y_{i_k}|}{p_{i_k}} - \frac{|x_{i_k}^\top \beta^* - z_1p_{i_k} - y_{i_k}|}{p_{i_k}}\right|\\
&\le \frac{|z_1p_{i_k} - z_2p_{i_k}|}{p_{i_k}} \le |z_1-z_2|
\end{align*}
Applying Lemma \ref{thm:talagrand} with these parameters, we have
\begin{align*}\mathop\mathbb{E}_{\sigma}&\left[\left(\frac{1}{2}\max_{\Vert X\beta^* - X\beta \Vert = 1}\left| \sum_{k} \sigma_k\left(\frac{|x_{i_k}^\top \beta^* - y_{i_k}|}{p_{i_k}} -  \frac{|x_{i_k}^\top \beta -true y_{i_k}|}{p_{i_k}}\right)\right|\right)^l\right]\nonumber\\&\le \frac{3}{2}\mathop\mathbb{E}_{ \sigma}\left[\left(\max_{\Vert X\beta^* - X\beta \Vert_1 = 1}\left| \sum_{k} \sigma_k\left(\frac{x_{i_k}^\top \beta^* - x_{i_k}^\top \beta}{p_{i_k}}\right)\right|\right)^l\right]\end{align*}
After taking the expectation with respect to $S$ and multiplying both sides by $2^l$, this gives the statement of the lemma. 
\end{proof}
From here, we use two separate results to show the appropriate row
counts for the constant and high probability cases. The constant
probability case is left for Appendix~\ref{app:constantfailure}.

For high probability row-counts, we use a lemma from \cite{cohen2014ellp}:
\begin{lemma}[8.2, 8.3, 8.4 in \cite{cohen2014ellp}]\label{lemma:talhigh}
There exists constant $C$ such that for any $X\in\mathbb{R}^{n\times d}$ with all $\ell_1$ Lewis weights less than $C \frac{\varepsilon^2}{\log\left(\frac{n}{\delta}\right)}$ and $l = \log(2n/\delta)$, then
\begin{align}\mathbb{E}_{\sigma}\left[\left(\max_{\Vert X\beta\Vert_1= 1}\left|\sum_{i = 1}^n\sigma_ix_{i}^\top \beta\right|\right)^l\right]\le \frac{\varepsilon^l\delta}{2}\end{align}
\end{lemma}
We want a similar statement, but for arbitrary matrices, with no bounds placed on the Lewis weights. To do this, we construct a new, related matrix using the following lemma, which is proved in Appendix~\ref{app:proofs}:

\begin{comment}Note that
\[\mathbb{E}_{\sigma}\left[\left(\max_{\Vert X\beta\Vert_1= 1}\sum_i\sigma_ix_{i}^\top \beta\right)^l\right] = \mathbb{E}_{\sigma}\left[\left(\max_{\Vert X\beta\Vert_1= 1}\left|\sum_i\sigma_ix_{i}^\top \beta\right|\right)^l\right]\]
since for any $\beta$ that satisfies $\Vert X\beta\Vert_1 = 1$, we have $-\beta$ also satisfying it, and so for a fixed set of $\sigma_i$'s, the expression on the right is symmetric around $0$. So, when you add the absolute values in, the maximum is not changed. \end{comment}
\begin{restatable}[Similar to \cite{cohen2014ellp} Lemma B.1]{lemma}{bonemod}\label{lemma:b1mod}
Let $X$ be any matrix, and let $W$ be the matrix that has the Lewis weights of $X$ in the diagonal entries. Let $N\ge \frac{d}{\varepsilon^2}\log \frac{d}{\varepsilon\delta}$. There exist constants $C_1, C_2, C_3$ such that we can construct a matrix $X'$ such that
\begin{itemize}
\item $X'$ has $C_1dN$ rows, 
\item $X'^\top W'^{-1}X'\succeq X^\top W^{-1}X$, (where $W'$ is the matrix that has the Lewis weights of $X'$ in the diagonal entries),
\item $\Vert X'\beta\Vert_1\le C_2\Vert X\beta\Vert_1$ for all $\beta$,
\item the Lewis weights of $X'$ are bounded by $\frac{C_3}{N}$.
\end{itemize}
\end{restatable}
\begin{lemma}\label{lemma:high}
Consider $X\in \mathbb{R}^{n\times d}$ with $\ell_1$ Lewis weights $w_i$. Let $p_i$ be some set of sampling values such that $N = \sum_i p_i$ and, for some constants $C, C_1, C_4$, 
    \[p_i\ge \frac{\log\left(\frac{N + C_1Nd}{\delta}\right)}{C\varepsilon^2} w_i\]
Then, if $N \ge C_4 \frac{d}{\varepsilon^2}\log\frac{d}{\varepsilon\delta}$ and if $S \sim \mathcal{S}(\{p_i\}_{i \in [n]})$, then
\begin{align}
    \mathop\mathbb{E}_{S, \sigma}\left[\left(\max_{\Vert X\beta \Vert_1 = 1}\left| \sum_{k=1}^{N} \sigma_k\frac{x_{i_k}^\top \beta}{p_{i_k}}\right|\right)^l\right] \le \frac{\varepsilon^l\delta}{2} \label{eq:const}
\end{align}
\end{lemma}

\begin{proof}[Proof of Lemma \ref{lemma:high}]
Ideally the Lewis weights of $SX$ would be bounded by $C\frac{\varepsilon^2}{\log \frac{N}{\delta}}$ and we could directly apply Lemma \ref{lemma:talhigh} to $SX$ to obtain a bound on the moment. However, we do not know this. Instead, we first construct $X'$ using $X$ as described in Lemma \ref{lemma:b1mod}. We then construct a new matrix $X''$ by stacking $X'$ on top of $SX$. Define $W''$ to be the diagonal matrix consisting of the $\ell_1$ Lewis weights of $X''$.  Define, for convenience, $R = N + C_1Nd$, which is the number of rows $X''$ has.

We can bound the term on the left side of (\ref{eq:const}) by the same term, summing over the rows of $X''$ instead. That is, 
    \[\mathop\mathbb{E}_{S, \sigma}\left[\left(\max_{\Vert X\beta \Vert = 1}\left| \sum_{k=1}^{N} \sigma_k\frac{x_{i_k}^\top \beta}{p_{i_k}}\right|\right)^l\right]\le \mathop\mathbb{E}_{S, \sigma}\left[\left(\max_{\Vert X\beta \Vert = 1}\left| \sum_{i= 1}^{R} \sigma_ix''^\top _{i}\beta\right|\right)^l\right]\]
Our goal is to apply Lemma \ref{lemma:talhigh} to the right side. To do this, we need to show the correct bound on its Lewis weights, and then have the term be a maximum over $\Vert X''\beta\Vert_1 = 1$, rather than $\Vert X\beta\Vert_1 = 1$.

{
    \paragraph{Bounding the Lewis weights of $X''$.} By Lemma \ref{lemma:lewrows}, the $\ell_1$ Lewis weights of a matrix do not increase when more rows are added. So, the rows in $X''$ that are from $X'$ have Lewis weights that are bounded above by $C_3\frac{\varepsilon^2}{\log\left(\frac{d}{\varepsilon\delta}\right)}$. Further, 
    \begin{comment}\begin{align*}
    X''^\top W''^{-1}X'' &= \sum_{i=1}^{R} \frac{1}{w''_i}x''_i(x''_i)^\top \\
    &\succeq \sum_{i=1}^{R-N} \frac{1}{w''_k}x''_k(x''_k)^\top &&\text{since} \sum_{i=kC_1d^2+1}^N  \frac{1}{w''_i}x''_i(x''_i)^\top\succeq 0 \\
    &=\sum_{i=1}^{R-N} \frac{1}{w'_k}x'_k(x'_k)^\top\\
    &=k\sum_{i=1}^{C_1d^2} \frac{k}{w_k}\frac{x_k}{k}(\frac{x_k}{k})^\top
    = X^\top W^{-1}X.
    \end{align*}
    \end{comment}
    \begin{align*}
    X''^\top W''^{-1}X'' &= \sum_{i=1}^{R} \frac{1}{w''_i}x''_i(x''_i)^\top \\
    &\succeq \sum_{i=1}^{R-N} \frac{1}{w''_k}x''_k(x''_k)^\top &&\text{since} \sum_{i=kC_1d^2+1}^N  \frac{1}{w''_i}x''_i(x''_i)^\top\succeq 0 \\
    &=X'^\top W'^{-1}X' \succeq X^\top W^{-1}X.
    \end{align*}
    So, any row $y_i = x_i/p_i$ in $X''$ that is from $SX$ satisfies
    \begin{align*}
    w''^2_i = y_i^\top (X''^\top  W''^{-1}X'')^{-1}y_i &\le y_i^\top (X^\top  W^{-1}X)^{-1}y_i\\&= \frac{1}{p_i^2}x_i^\top (X^\top W^{-1}X)^{-1}x_i\\&\le \left( \frac{C\varepsilon^2}{ \log\left(\frac{R}{\delta}\right)}\frac{1}{w_i}\right)^2\cdot w_i^2 = \left( \frac{C\varepsilon^2}{ \log\left(\frac{R}{\delta}\right)}\right)^2
    \end{align*}
    which means that all of the Lewis weights of $X''$ are less than the larger of $C \frac{\varepsilon^2}{\log\left(\frac{R}{\delta}\right)}$ and $C_3\frac{\varepsilon^2}{\log\left(\frac{d}{\varepsilon\delta}\right)}$. Now, for small enough $\varepsilon, \delta$, $\log\frac{R}{\delta} \le \frac{C}{C_3}\log\frac{d}{\varepsilon\delta}$, we have the Lewis weight upper bound for all rows of $X''$ is $C \frac{\varepsilon^2}{\log\left(\frac{R}{\delta}\right)}$.
}

{

    \paragraph{Renormalizing to maximize over $\Vert X''\beta\Vert_1 = 1$: }If we define the following
    \[F \coloneqq  \max_{\Vert X\beta\Vert_1 = 1} \left| \Vert SX\beta\Vert_1 - \Vert X\beta\Vert_1\right|\]
    then, 
    \[\Vert X''\beta\Vert_1 = \Vert SX\beta\Vert_1 + \Vert X'\beta\Vert_1  \le (1 + C_2 + F)\Vert X\beta\Vert_1\]
    
    So, we get 
    \begin{align*}
    \left(\max_{\Vert X\beta \Vert = 1}\left| \sum_{k = 1}^{R} \sigma_kx''^\top _{k}\beta\right|\right)^l &\le (1 + C_2 + F)^l\left(\max_{\Vert X''\beta \Vert = 1}\left| \sum_{k = 1}^{R} \sigma_kx''^\top _{k}\beta\right|\right)^l\\
    &\le 2^{l-1}((1 + C_2)^l + F^l)\left(\max_{\Vert X''\beta \Vert = 1}\left| \sum_{k = 1}^{R} \sigma_kx''^\top _{k}\beta\right|\right)^l
    \end{align*}
    
    Taking expectations of either side over just the Rademacher variables,
    \begin{align*}
    \mathop\mathbb{E}_{\sigma}\left[\left(\max_{\Vert X\beta \Vert = 1}\left| \sum_{k = 1}^{R} \sigma_kx''^\top _{k}\beta\right|\right)^l\right] &\le 2^{l-1}((1 + C_2)^l + F^l)\mathop\mathbb{E}_{\sigma}\left[\left(\max_{\Vert X''\beta \Vert = 1}\left| \sum_{k = 1}^{R} \sigma_kx''^\top _{k}\beta\right|\right)^l\right]
    \end{align*}
    
    \paragraph{Applying Lemma \ref{lemma:talhigh} to $X''$: } Since $X''$ has $R$ rows, and the correct Lewis weight bound, we can simply apply Lemma \ref{lemma:talhigh} to the right side above
    \begin{align*}
    \mathop\mathbb{E}_{\sigma}\left[\left(\max_{\Vert X\beta \Vert = 1}\left| \sum_{k = 1}^{R} \sigma_kx''^\top _{k}\beta\right|\right)^l\right] &\le 2^{l-1}((1 + C_2)^l + F^l))\frac{\varepsilon^l\delta}{2}
    \end{align*}
    Now, by Lemma \ref{lemma:moment}, we know that $\mathbb{E}_S [F^l] \le \varepsilon^l\delta$. So, taking the expectation with respect to the sampling matrices of either side of the above, we get, for small enough $\varepsilon, \delta$, 
    \begin{align*}
    \mathop\mathbb{E}_{S, \sigma}\left[\left(\max_{\Vert X\beta \Vert = 1}\left| \sum_{k = 1}^{kC_1d^2 + N} \sigma_kx''^\top _{k}\beta\right|\right)^l\right] &\le 2^{l-1}((1+C_2)^l +\varepsilon^l\delta)\frac{\varepsilon^l\delta}{2}\le 2^l(1+C_2)^l\frac{\varepsilon^l\delta}{2}
    \end{align*}
    So, solving the problem for $\varepsilon' = \frac{\varepsilon}{2+ 2C_2}$ gives the correct bound. 
}

\end{proof}
Finally, we can show Lemma \ref{thm:mainlemma}
\begin{proof}[Proof of Lemma \ref{thm:mainlemma}]
Take $l = \log(2n/\delta)$,  $N = 5\frac{(1+C_1)C_3}{C}\frac{d}{\varepsilon^2} \log\frac{d}{\varepsilon\delta}$. Then, we apply Lemma \ref{lemma:prelim_symmetrize}, Lemma \ref{thm:triangle}, and Lemma \ref{lemma:high} to get 
\begin{align*}
    M \le 2^{2l}\varepsilon^l\delta
\end{align*}
which, solving the problem for $\varepsilon/4$, gives the correct bound. Then, applying Markov's inequality, we get that with probability $\delta$, 
\[\max_{\Vert X\beta^* - X\beta \Vert = 1}\left|\left(\Vert SX\beta^* - Sy\Vert_1 - \Vert SX\beta - Sy\Vert_1\right) - \left(\Vert X\beta^* - y\Vert_1 - \Vert X\beta - y\Vert_1\right)\right|\le \varepsilon\]
Finally, scaling up appropriately gives, in generality,
\[\left|\left(\Vert SX\beta^* - Sy\Vert_1 - \Vert SX\beta - Sy\Vert_1\right) - \left(\Vert X\beta^* - y\Vert_1 - \Vert X\beta - y\Vert_1\right)\right|\le \varepsilon\Vert X\beta^* - X\beta\Vert_1\]
\end{proof}

\bibliography{sources}
\bibliographystyle{alpha}
\pagebreak
\appendix

\section{Constant-factor approximation}\label{app:constant-factor}
If we just want a constant factor approximation, we can take $S$ to be a constant probability $\ell_1$-subspace embedding, so that $\Vert X\beta\Vert_1 \le 2\Vert SX\beta\Vert_1$ with probability at least $0.9$. We have
\begin{align*}
    \Vert X\wh\beta - y\Vert_1 &\le \Vert X\wh\beta - X\beta^*\Vert_1 + \Vert X\beta^* - y\Vert_1\\
    &\le 2\Vert SX\wh\beta - SX\beta^*\Vert_1 + \Vert X\beta^* - y\Vert_1\\
    &\le 2(\Vert SX\wh\beta - Sy\Vert_1 + \Vert SX\beta^*- Sy\Vert_1) + \Vert X\beta^* - y\Vert_1\\
    &\le 4(\Vert SX\beta^*- Sy\Vert_1) + \Vert X\beta^* - y\Vert_1
\end{align*}
where in the last inequality, we have used the fact that $\wh\beta$ is the minimizer of $\Vert SX\beta - Sy\Vert_1$. Now, by Markov's inequality, with probability 0.9, $\Vert SX\beta^* - Sy\Vert_1 \le 10 \Vert X\beta^*- y\Vert_1$. So, we have with probability $0.81$, 
\[\Vert X\wh\beta - y\Vert_1 \le 41 \Vert X\beta^* - y\Vert_1\]
Since we only used a constant-factor subspace embedding, the row count would be $O(d\log d)$. 

\section{Proofs of Lemmas}\label{app:proofs}
\lewbound*
\begin{proof}

Let $\{w_i\}_{i = 1}^n$ be the Lewis weights of $X$, and let $\{w'_i\}_{i = 1}^{kn}$ be the Lewis weights of $X'$. Let $x_i$ be the $i$th row of $X$, and similarly let $x'_i$ be the $i$th row of $X'$. Let the ordering of the rows be such that $x'_{jn+i} = \frac{1}{k}x_i$ for $0 \le j < k$. Let $W$ be the diagonal matrix where $W_{ii} = w_i$. Since Lewis weights are defined circularly, we just need to check that the suggested weights work, and by uniqueness, they will be correct.

We know that $w_i^2 = x_i^\top (X^\top W^{-1}X)^{-1}x_i$. Therefore, if we take $W'$ to be the diagonal matrix of size $kn\times kn$, and set the diagonal entries to be the Lewis weights of $X$ divided by $k$, repeated $k$ times, then we have 
\begin{align*}
X'^\top W'^{-1}X' &= \sum_{i = 1}^{kn}\frac{1}{w'_i}x'_ix'^\top _i = \sum_{i = 1}^{kn}\frac{k}{w_i}x'_ix'^\top _i = k\sum_{i = 1}^{n}\frac{k}{w_i}\cdot \frac{1}{k^2}x_ix_i^\top 
\end{align*}
In the last expression above, we are only summing over the first set of rows in $X'$, which are the scaled rows of $X$, and then multiplying by $k$ since they are repeated $k$ times. Now,
\begin{align*}
k\sum_{i = 1}^{n}\frac{k}{w_i}\cdot \frac{1}{k^2}x_ix_i^\top  &= \sum_{i=1}^n \frac{1}{w_i}x_ix_i^\top  = X^\top W^{-1}X
\end{align*}
So, finally, for an arbitrary row $x'_{jn+i}$, which corresponds to row $x_i$ in the original matrix, we get its Lewis weight:
\[w_{jn+i}'^2 =x'^\top _{jn+i}(X'^\top W'^{-1}X')^{-1}x'_{jn+i} = \frac{1}{k^2}x^\top _i(X^\top W^{-1}X)^{-1}x_i = \frac{w_i^2}{k^2}\]
which proves that our suggested Lewis weights are consistent.
\end{proof}
\prelimsymmetrize*
\begin{proof}
We proceed by symmetrization. Since the matrix $S$ scales the rows by the probability they are picked with, the expectation of $\Vert SM\beta\Vert_1$ is just $\Vert M\beta \Vert_1$, for any matrix $M$ and vector $\beta$. So, adding or subtracting the same term with a different sampling matrix $S'$, $\left(\Vert S'X\beta^* - S'y\Vert_1 - \Vert S'X\beta - S'y\Vert_1\right) - \left(\Vert X\beta^* - y\Vert_1 - \Vert X\beta - y\Vert_1\right)$, is just adding a mean zero term, and since taking the $l$th power of a maximum is convex, this can only increase the expectation. That is,
\begin{align*}
    &\mathop\mathbb{E}_{S, S'}\Bigg[\bigg(\max_{\Vert X\beta^* - X\beta \Vert = 1}|\left(\Vert SX\beta^* - Sy\Vert_1 - \Vert SX\beta - Sy\Vert_1\right) - \left(\Vert X\beta^* - y\Vert_1 - \Vert X\beta - y\Vert_1\right)|\bigg)^l\Bigg]
    \\&\le\mathop\mathbb{E}_{S, S'}\Bigg[\bigg(\max_{\Vert X\beta^* - X\beta \Vert = 1}|\left(\left(\Vert SX\beta^* - Sy\Vert_1 - \Vert SX\beta - Sy\Vert_1\right) - \left(\Vert X\beta^* - y\Vert_1 - \Vert X\beta - y\Vert_1\right)\right) \\&\qquad\qquad\qquad-\left(\left(\Vert S'X\beta^* - S'y\Vert_1 - \Vert S'X\beta - S'y\Vert_1\right) - \left(\Vert X\beta^* - y\Vert_1 - \Vert X\beta - y\Vert_1\right)\right)|\bigg)^l\Bigg]
\end{align*}
 So, we can bound $M$ as
\begin{align*}
    M\le \mathop\mathbb{E}_{S, S'}\Bigg[\bigg(\max_{\Vert X\beta^* - X\beta \Vert = 1}&|\left(\Vert SX\beta^* - Sy\Vert_1 - \Vert SX\beta - Sy\Vert_1\right) - \\&\left(\Vert S'X\beta^* - S'y\Vert_1 - \Vert S'X\beta - S'y\Vert_1\right)|\bigg)^l\Bigg]
\end{align*}
Let $i_k$ be the indices chosen by $S$, and $i_k'$ the indices chosen by $S'$. Rewriting this as a sum,
\begin{align*}
M \le \mathop\mathbb{E}_{S, S'}\Bigg[\bigg(\max_{\Vert X\beta^* - X\beta \Vert = 1}\bigg| &\sum_{k} \left(\frac{|x_{i_k}^\top \beta^* - y_{i_k}|}{p_{i_k}} -  \frac{|x_{i_k}^\top \beta - y_{i_k}|}{p_{i_k}}\right) -\\& \sum_k \left(\frac{|x_{i'_k}^\top \beta^* - y_{i'_k}|}{p_{i'_k}} - \frac{|x_{i'_k}^\top \beta - y_{i'_k}|}{p_{i'_k}}\right)\bigg|\bigg)^l\Bigg]
\end{align*}

Now, since $i_k$ and $i_k'$ are independent and identically distributed, randomly swapping elements from either sum does not change the distribution. This amounts to adding a random sign $\sigma_k$ to the terms, where $\sigma_k=\pm 1$ independently with probability $1/2$. So, 

\begin{align*}
    M \le \mathop\mathbb{E}_{S, S', \sigma}\Bigg[\bigg(\max_{\Vert X\beta^* - X\beta \Vert = 1}\bigg| &\sum_{k} \sigma_k\left(\frac{|x_{i_k}^\top \beta^* - y_{i_k}|}{p_{i_k}} -  \frac{|x_{i_k}^\top \beta - y_{i_k}|}{p_{i_k}}\right) -\\& \sum_k \sigma_k\left(\frac{|x_{i'_k}^\top \beta^* - y_{i'_k}|}{p_{i'_k}} - \frac{|x_{i'_k}^\top \beta - y_{i'_k}|}{p_{i'_k}}\right)\bigg|\bigg)^l\Bigg]\\
    \le \mathop\mathbb{E}_{S, S', \sigma}\Bigg[\bigg(\max_{\Vert X\beta^* - X\beta \Vert = 1}\bigg| &\sum_{k} \sigma_k\left(\frac{|x_{i_k}^\top \beta^* - y_{i_k}|}{p_{i_k}} -  \frac{|x_{i_k}^\top \beta - y_{i_k}|}{p_{i_k}}\right)\bigg| +\\& \max_{\Vert X\beta^* - X\beta \Vert = 1}\bigg|\sum_k \sigma_k\left(\frac{|x_{i'_k}^\top \beta^* - y_{i'_k}|}{p_{i'_k}} - \frac{|x_{i'_k}^\top \beta - y_{i'_k}|}{p_{i'_k}}\right)\bigg|\bigg)^l\Bigg]\\
    \le 2^l\mathop\mathbb{E}_{S, \sigma}\Bigg[\bigg(\max_{\Vert X\beta^* - X\beta \Vert = 1}\bigg| &\sum_{k} \sigma_k\left(\frac{|x_{i_k}^\top \beta^* - y_{i_k}|}{p_{i_k}} -  \frac{|x_{i_k}^\top \beta - y_{i_k}|}{p_{i_k}}\right)\bigg|\bigg)^l\Bigg]
\end{align*}
Where the final inequality follows from $(a+b)^l\le 2^{l-1}(a^l+b^l)$. Putting these together,
\begin{align}M \le  2^l\mathop\mathbb{E}_{S, \sigma}\left[\left(\max_{\Vert X\beta^* - X\beta \Vert = 1}\left| \sum_{k} \sigma_k\left(\frac{|x_{i_k}^\top \beta^* - y_{i_k}|}{p_{i_k}} -  \frac{|x_{i_k}^\top \beta - y_{i_k}|}{p_{i_k}}\right)\right|\right)^l\right]\label{eq:4}\end{align}
\end{proof}
\bonemod*
\begin{proof}
Given matrix $X$, we can use Lemma B.1 from \cite{cohen2014ellp} to construct a new matrix $X_1$ that satisfies
\begin{itemize}
\item $X_1$ has $C_1d^2$ rows, 
\item $X_1^\top W_1^{-1}X_1\succeq X^\top W^{-1}X$, (where $W_1$ is the matrix that has the Lewis weights of $X_1$ in the diagonal entries),
\item $\Vert X_1\beta\Vert_1\le C_2\Vert X_1\beta\Vert_1$ for all $\beta$,
\item the Lewis weights of $X_1$ are bounded by $\frac{C_3}{d}$. 
\end{itemize}
So, we can take this matrix and stack it on itself $k = \frac{N}{d}$ times, while scaling each row down by the same $k$. This will be our matrix $X'$. $X'$ will then have $k = C_1Nd$ rows, which satisfies the first bullet. Also, by Lemma \ref{lemma:lewbound}, this shrinks the Lewis weights by a factor of $k$, which changes the Lewis weight upper bound to 
\[\frac{C_3}{kd} = \frac{C_3}{N}\]
which is what we need. Now, since we are repeating rows $k$ times, but each row is scaled down by $k$, we have $\Vert X_1\beta\Vert_1 = \Vert X'\beta\Vert_1$ for all $\beta$. Therefore, $\Vert X'\beta\Vert_1 \le C_2\Vert X\beta\Vert_1$ for all $\beta$. Finally, as in the proof of Lemma \ref{lemma:lewbound}, we know that since we have duplicated the rows of $X_1$ $k$ times but scaled them down by $k$, $X_1^\top W_1^{-1}X_1 = X'^\top W'^{-1}X'$, and so we are done.
\end{proof}
\section{Proof for constant failure probability}\label{app:constantfailure}
For the constant probability row-count, we use a lemma from \cite{ledoux1989}: 
\begin{lemma}[\cite{ledoux1989}]\label{lemma:talconst}
There exists a constant $C$ such that for any matrix $X$ with all Lewis weights less than $C\frac{\varepsilon^2}{\log d}$, 
\[\mathop\mathbb{E}_{\sigma}\left[\max_{\Vert X\beta\Vert_1= 1}\sum_k\sigma_kx_{i}^\top \beta\right]\le \varepsilon\]
\end{lemma}
In \cite{ledoux1989}, this is proved with absolute values within the sum (that is, summing $\sigma_i|x_i^\top \beta|$). However, the first step of the proof removes these absolute values using a comparison lemma, bounding the term with absolute values by twice the term without absolute values. 

\begin{lemma}\label{lemma:const}
For matrix $X$ with $\ell_1$ Lewis weights $w_i$, let $p_i$ be some set of sampling values such that $\sum_i p_i = N$ and $p_i\ge \frac{\log d}{C\varepsilon^2} w_i$. If you sample $S\sim \mathcal{S}(\{p_i\}_{i \in [n]})$, then
\begin{align}
    \mathop\mathbb{E}_{S, \sigma}\left[\max_{\Vert X\beta \Vert_1 = 1}\left| \sum_{k} \sigma_k\frac{x_{i_k}^\top \beta}{p_{i_k}}\right|\right] \le \varepsilon\label{eq:const1}
\end{align}
\end{lemma}
\begin{proof}
This proof is very similar to that of Lemma \ref{lemma:high}.

Construct $X'$ using $X$ as described in Lemma \ref{lemma:b1mod}, with $N = \frac{C_3}{C}\frac{\log d}{\varepsilon^2}$. We then construct a new matrix $X''$ by stacking $X'$ on top of $SX$. Define $W''$ to be the diagonal matrix consisting of the $\ell_1$ Lewis weights of $X''$. 

We can bound the term on the left side of (\ref{eq:const1}) by the same term, summing over the rows of $X''$ instead. That is, 
    \[\mathop\mathbb{E}_{S, \sigma}\left[\max_{\Vert X\beta \Vert = 1}\left| \sum_{k=1}^{N} \sigma_k\frac{x_{i_k}^\top \beta}{p_{i_k}}\right|\right]\le \mathop\mathbb{E}_{S, \sigma}\left[\max_{\Vert X\beta \Vert = 1}\left| \sum_{i= 1}^{R} \sigma_ix''^\top _{i}\beta\right|\right]\]
Our goal is to apply Lemma \ref{lemma:talconst} to the right side. To do this, we need to show the correct bound on its Lewis weights, and then have the term be a maximum over $\Vert X''\beta\Vert_1 = 1$, rather than $\Vert X\beta\Vert_1 = 1$.

{
    \paragraph{Bounding the Lewis weights of $X''$.} By Lemma \ref{lemma:lewrows}, the $\ell_1$ Lewis weights of a matrix do not increase when more rows are added. So, the rows in $X''$ that are from $X'$ have Lewis weights that are bounded above by $\frac{C\varepsilon^2}{\log d}$. Further, 
    \begin{comment}\begin{align*}
    X''^\top W''^{-1}X'' &= \sum_{i=1}^{R} \frac{1}{w''_i}x''_i(x''_i)^\top \\
    &\succeq \sum_{i=1}^{R-N} \frac{1}{w''_k}x''_k(x''_k)^\top &&\text{since} \sum_{i=kC_1d^2+1}^N  \frac{1}{w''_i}x''_i(x''_i)^\top\succeq 0 \\
    &=\sum_{i=1}^{R-N} \frac{1}{w'_k}x'_k(x'_k)^\top\\
    &=k\sum_{i=1}^{C_1d^2} \frac{k}{w_k}\frac{x_k}{k}(\frac{x_k}{k})^\top
    = X^\top W^{-1}X.
    \end{align*}
    \end{comment}
    \begin{align*}
    X''^\top W''^{-1}X'' &= \sum_{i=1}^{R} \frac{1}{w''_i}x''_i(x''_i)^\top \\
    &\succeq \sum_{i=1}^{R-N} \frac{1}{w''_k}x''_k(x''_k)^\top &&\text{since} \sum_{i=kC_1d^2+1}^N  \frac{1}{w''_i}x''_i(x''_i)^\top\succeq 0 \\
    &=X'^\top W'^{-1}X' \succeq X^\top W^{-1}X.
    \end{align*}
    So, any row $y_i = x_i/p_i$ in $X''$ that is from $SX$ satisfies
    \begin{align*}
    w''^2_i = y_i^\top (X''^\top  W''^{-1}X'')^{-1}y_i &\le y_i^\top (X^\top  W^{-1}X)^{-1}y_i\\&= \frac{1}{p_i^2}x_i^\top (X^\top W^{-1}X)^{-1}x_i\\&\le \left( \frac{C\varepsilon^2}{ \log d}\frac{1}{w_i}\right)^2\cdot w_i^2 = \left( \frac{C\varepsilon^2}{ \log d}\right)^2
    \end{align*}
    which means that all of the Lewis weights of $X''$ are less than $\frac{C\varepsilon^2}{ \log d}$.
}

{

    \paragraph{Renormalizing to maximize over $\Vert X''\beta\Vert_1 = 1$: }If we define the following
    \[F \coloneqq  \max_{\Vert X\beta\Vert_1 = 1} \left| \Vert SX\beta\Vert_1 - \Vert X\beta\Vert_1\right|\]
    then, 
    \[\Vert X''\beta\Vert_1 = \Vert SX\beta\Vert_1 + \Vert X'\beta\Vert_1  \le (1 + C_2 + F)\Vert X\beta\Vert_1\]
    
    So, we get 
    \begin{align*}
    \max_{\Vert X\beta \Vert = 1}\left| \sum_{k = 1}^{R} \sigma_kx''^\top _{k}\beta\right| &\le (1 + C_2 + F)\cdot \max_{\Vert X''\beta \Vert = 1}\left| \sum_{k = 1}^{R} \sigma_kx''^\top _{k}\beta\right| 
    \end{align*}
    
    Taking expectations of either side over just the Rademacher variables,
    \begin{align*}
    \mathop\mathbb{E}_{\sigma}\left[\max_{\Vert X\beta \Vert = 1}\left| \sum_{k = 1}^{R} \sigma_kx''^\top _{k}\beta\right|\right] &\le (1+C_2+F)\mathop\mathbb{E}_{\sigma}\left[\max_{\Vert X''\beta \Vert = 1}\left| \sum_{k = 1}^{R} \sigma_kx''^\top _{k}\beta\right|\right]
    \end{align*}
    
    \paragraph{Applying Lemma \ref{lemma:talconst} to $X''$: } Since $X''$ has $R$ rows, and the correct Lewis weight bound, we can simply apply Lemma \ref{lemma:talconst} to the right side above
    \begin{align*}
    \mathop\mathbb{E}_{\sigma}\left[\max_{\Vert X\beta \Vert = 1}\left| \sum_{k = 1}^{R} \sigma_kx''^\top _{k}\beta\right|\right] &\le (1 + C_2 + F)\varepsilon
    \end{align*}
    Now, by Lemma \ref{lemma:moment}, we know that $\mathbb{E}_S [F] \le \varepsilon$. So, taking the expectation with respect to the sampling matrices of either side of the above, we get, for small enough $\varepsilon$, 
    \begin{align*}
    \mathop\mathbb{E}_{S, \sigma}\left[\max_{\Vert X\beta \Vert = 1}\left| \sum_{k = 1}^{kC_1d^2 + N} \sigma_kx''^\top _{k}\beta\right|\right] &\le 2(1+C_2)\varepsilon 
    \end{align*}
    So, solving the problem for $\varepsilon' = \frac{\varepsilon}{2+ 2C_2}$ gives the correct bound. 
}

\end{proof}
Therefore, we can similarly prove the constant-probability case for Lemma \ref{thm:mainlemma}:
\begin{proof}[Proof of \ref{thm:mainlemma} for constant probability]
We take $l = 1$, $N = \frac{d}{\varepsilon^2}\log d$ and apply Lemma \ref{lemma:prelim_symmetrize}, Lemma \ref{thm:triangle}, and Lemma \ref{lemma:const}. 

\end{proof}

\section{Lower Bounds}
We prove three main theorems that allow us to show Theorem \ref{lower_bound}: Theorems \ref{lower_bound_1}, \ref{lower_bound_2}, and \ref{lower_bound_3}. To do this, we make several Claims, which are proved in section 7.1.
Recall the reduction between the matrix problem and the distribution:
\distr*
\begin{proof}
Let $n = \frac{8}{\varepsilon^2}\left(\log \frac{2}{\delta} + d\log \frac{4d}{\varepsilon}\right)$. Construct an instance of Problem \ref{matrix_problem} in which the rows of feature matrix $\mathbf{X}$ and the corresponding label vector $y$ are drawn i.i.d. from $P$. Let $H$ be the unit $\ell_\infty$ ball. We have the following:

\begin{restatable}{claim}{hoeffding}\label{claim:hoeffding}
For all $\beta \in H$, with probability at least $1 - \delta$, 
% Using Hoeffding and a union bound over an $\varepsilon$-net, we have, for all $\varepsilon$
\begin{align*}
    (1-\varepsilon)\E_{(X, Y)\sim P}\left[|X^\top \beta - Y|\right] \le \frac{1}{n}\Vert \mathbf{X}\beta-y\Vert_1 \le (1+\varepsilon)\E_{(X, Y)\sim P}\left[|X^\top \beta - Y|\right]
\end{align*}
\end{restatable}
Let $\beta^\circ$ denote the minimizer $\inf_\beta \E_{(X, Y)\sim P}\left[|X^\top \beta - Y|\right]$. Let $\beta^*$ denote the minimizer of the matrix instance $\inf_\beta \Vert\mathbf{X}\beta-y\Vert_1$, and let $\wh\beta$ denote the output of the algorithm on the instance generated. Then we have 
\begin{align*}
    (1-\varepsilon)\E_{(X, Y)\sim P}\left[|X^\top \wh\beta - Y|\right] 
    &\le \frac{1}{n}\Vert \mathbf{X}\wh\beta-y\Vert_1\\
    &\le (1+\varepsilon)\frac{1}{n}\Vert \mathbf{X}\beta^*-y\Vert_1 && \text{with probability $1-\delta$}\\
    &\le (1+\varepsilon)\frac{1}{n}\Vert \mathbf{X}\beta^\circ-y\Vert_1\\
    &\le (1+\varepsilon)^2\E_{(X, Y)\sim P}\left[|X^\top \beta^\circ - Y|\right]
\end{align*}
So with probability $1-2\delta$, $$\E_{(X, Y)\sim P}\left[|X^\top \wh\beta - Y|\right]\le (1+6\varepsilon)\E_{(X, Y)\sim P}\left[|X^\top \beta^\circ - Y|\right].$$
\end{proof}
\begin{theorem}\label{lower_bound_1}
For any $d\ge 2$ and $\varepsilon<\frac{1}{10}$, there exist families $\mathcal{X}\in \mathbb{R}^d, \mathcal{Y}\in\mathbb{R}$ of inputs and labels respectively such that any algorithm which solves Problem \ref{distribution_problem} with $\delta<\frac{1}{4}$ requires at least $m=\frac{3d}{2000\varepsilon^2}$ samples.
\end{theorem}

We take $\mathcal{X}$ to be the set of standard basis vectors, and the distribution over $\mathcal{X}$ to be uniform. We will define a set $\mathcal{B}$ as being a subset of the unit hypercube $\{-1, 1\}^d$ such that every element is sufficiently far from every other.
\begin{claim}
There is a set $\mathcal{B}\subset \mathcal{H}$ with $|\mathcal{B}| \ge 2^{0.2d}$ such that for any two $\beta_1, \beta_2\in \mathcal{B}$, we have $|\beta_1-\beta_2|>0.2d$
\end{claim}
\begin{proof}
Here we just need an error correcting code with constant rate and constant relative (Hamming) distance. The existence of such a code follows from the Gilbert-Varshamov bound \cite{gilbert}.
\end{proof}
Fix some unknown $\beta^*$. We will have $Y=ZX^\top\beta^*$ where $Z$ is an independent random variable with probability $\frac{1}{2}+\varepsilon$ of being $1$, and $\frac{1}{2}-\varepsilon$ of being $-1$. This completes our description of $P$.  We define $l(\beta)$ to be the $\ell_1$ norm of the residuals for $\beta$, that is, $l(\beta)=\E_{(X, Y)\sim P} [\big|X^\top \beta-Y\big|]$. We have the following properties of $l(\beta)$.
\begin{restatable}{claim}{lossone}\label{claim:loss1}
For $D, \mathcal{B}$ as chosen above, $l(\beta^*)=1-2\varepsilon$.
\end{restatable}
\begin{restatable}{claim}{losstwo}\label{claim:loss2}
For $D, \mathcal{B}$ as chosen above, we have for all $\beta \in \mathcal{B}$, $l(\beta)-l(\beta^*) = \frac{2\varepsilon}{d}||\beta-\beta^*||_1$.
\end{restatable}
\begin{proof}[Proof of Theorem \ref{lower_bound_1}]
Suppose some algorithm returns $\wh{\beta}$ with $l(\wh{\beta})<(1+\frac{\varepsilon}{5}) l(\beta^*) \implies ||\beta^*-\wh{\beta}||_1<0.1d$ with probability $\frac{3}{4}$. By Fano's inequality, $$H(\beta^*|\wh{\beta})<H\left(\frac{1}{4}\right)+\frac{\log|\mathcal{B}|-1}{4}<0.05d,$$ and we have a lower bound on the mutual information between the output of our algorithm and the true parameter: $I(\wh{\beta}; \beta^*)=H(\beta^*)-H(\beta^*|\wh{\beta})\ge 0.15d$. For an upper bound on the mutual information after seeing $m$ samples, we use the data processing inequality.
\begin{align*}
I(\beta^*; \wh{\beta}) 
&\le I(\beta^*; (Y_i)_{i\in [m]}) \le \sum_{i=1}^m I(\beta^*; Y_i|(Y_j)_{j\in [i-1]})\\
& = \sum_{i=1}^m H(Y_i|(Y_j)_{j\in [i-1]})-H(Y_i|\beta^*, (Y_j)_{j\in [i-1]})\\
& \le \sum_{i=1}^m 1-H(Y_i|\beta^*, I_i)\\
&\le 4\varepsilon^2m
\end{align*}

Here we have used that
\begin{align*}
    H(Y_i|\beta^*, (Y_j)_{j\in [i-1]}) 
    &\ge H(Y_i|\beta^*, I_i, (Y_j)_{j\in [i-1]}) \\
    &= H(Y_i|\beta^*, I_i)
\end{align*}
and that the distribution of $Y_i$ conditioned on $\beta^*, I_i$ is just an independent Bernoulli with parameter $\frac{1}{2}+\varepsilon$ and so
\begin{align*}
    \sum_{i=1}^m 1-H(Y_i|\beta^*, I_i)
    &\le \sum_{i=1}^m\left[ 1+\left(\frac{1}{2}+\varepsilon\right)\log\left(\frac{1}{2}+\varepsilon\right)+\left(\frac{1}{2}-\varepsilon\right)\log\left(\frac{1}{2}-\varepsilon\right)\right]\\
    &\le 4\varepsilon^2m
\end{align*}
So $0.15d\le I(\beta^*; \wh{\beta}) \le 4\varepsilon^2m$, and so we need $m\ge \frac{3d}{80\varepsilon^2}$. The result follows by replacing $\varepsilon$ with $5\varepsilon$.
\end{proof}

We can use the same instance to give a high probability lower bound of $\Omega(\log\frac{1}{\delta}/\varepsilon^2)$.
\begin{theorem}\label{lower_bound_2}
For any $d$ and $\varepsilon<\frac{1}{10}$, there exist sets $\mathcal{X}\in \mathbb{R^d}, \mathcal{Y}\in \mathbb{R}$ of inputs and labels respectively, and a distribution $P$ on $\mathcal{X}\times \mathcal{Y}$ such that any algorithm which solves problem \ref{distribution_problem} requires at least $m=\frac{1}{4\varepsilon^2}\log\frac{1}{\delta}$ samples.
\end{theorem}

\begin{proof}
Consider two instances, denoted by subscripts (1) and (2) with $\beta^*_{(1)} = -\mathbbm{1}_d$ and $\beta^*_{(2)} = \mathbbm{1}_d$, where $\mathbbm{1}_d \in \mathbb{R}^d$ is the all-ones vector. Denote by $P_{(i)}$ the distribution over $\mathcal{X}, \mathcal{Y}$ for instance $(i)$, and let $l_{\beta^*_{(i)}}(\beta) = \E_{(X, Y)\sim P_{(i)}} [\big|X^\top \beta-Y\big|]$ for $i \in \{1, 2\}$. 
\begin{restatable}{claim}{losstwofive}\label{claim:loss2.5}
For any $\beta$, $\max \{\ell_{\beta^*_{(1)}}(\beta)-\ell_{\beta^*_{(1)}}(\beta^*_{(1)}), \ell_{\beta^*_{(2)}}(\beta)-\ell_{\beta^*_{(2)}}(\beta^*_{(2)})\}>2\varepsilon$
\end{restatable}
From this claim together with Claim \ref{claim:loss1}, we have for some $i \in \{1, 2\}$,  $l_{\beta^*_{(i)}}(\beta) \ge (1+2\varepsilon) l_{\beta^*_{(i)}}(\beta^*_{(i)})$, for all $\beta$.

Denote by $\wh{\beta}$ the output of the algorithm. Denote by $\mathbb{P}_{(1)}$ the distribution over outputs by a algorithm interacting instance $(1)$, and by $\mathbb{P}_{(2)}$ the distribution over outputs by a algorithm interacting instance $(2)$. Denote by $A$ the event that $\ell_{\beta^*_{(1)}}(\wh\beta)-\ell_{\beta^*_{(1)}}(\beta^*_{(1)}) \ge 2\varepsilon$. Note that under $A^c$, we have $\ell_{\beta^*_{(2)}}(\wh\beta)-\ell_{\beta^*_{(2)}}(\beta^*_{(2)}) \ge 2\varepsilon$. Because the algorithm fails with probability at most $\delta$ on any instance, we have $2\delta\ge \mathbb{P}_{(1)}(A)+\mathbb{P}_{(2)}(A^c)$. On the other hand, $\mathbb{P}_{(1)}(A)+\mathbb{P}_{(2)}(A^c)\ge e^{-D(\mathbb{P}_{(1)}|| \mathbb{P}_{(2)})}$. We can bound the KL-divergence of the two distributions as an aggregate KL-divergence over the course of acquiring the samples. 

% Factorize $\mathbb{P}_{(1)}$ (respectively $\mathbb{P}_{(2)}$) as $\prod_{j=1}^m\pi(I_j|I_1, Y_1, \cdots, Y_{j-1})P_{(1)}(Y_j | I_j)$ where $\pi(\cdot|I_1, Y_1, \cdots, Y_{j-1})$ denotes the distribution over rows sampled by the algorithm at the $j$th time conditioned on the history $(I_1, Y_1, I_2,\cdots, Y_{j-1})$. Because $\pi$ is a property of the algorithm and is the same for both instances, we have the following decomposition of KL-divergences, as shown in Theorem 15.1 of \cite{LatSze20} which essentially relies on the observation that the terms corresponding to the choice of row samples in the probability of seeing a particular sample path are identical for both instances conditioned on history. 
\begin{theorem}[Lemma 15.1, \cite{LatSze20}]
    If a learner interacts with two environments $(1)$ and $(2)$ through a policy $\pi(\cdot|I_1, Y_1, I_2, Y_2,\cdots, Y_{i-1})$ which dictates a distribution over actions $I_i$ conditioned on the past $(I_1, Y_1, \cdots, Y_{i-1})$, and sees label $Y_i$ distributed according to some label distribution $P_{(1), I_i}$ and $P_{(2), I_i}$, then the KL-divergence between the output of the learner on instance $(1)$ and $(2)$, $\mathbb{P}_{(1)}$ and $\mathbb{P}_{(2)}$ is given by 
    $$D(\mathbb{P}_{(1)}||\mathbb{P}_{(2)}) = \sum_{k=1}^d \E_{(1)}\left[\sum_{i=1}^N \mathbbm{1}\{I_i = k\}\cdot D(P_{(1), I_i}||P_{(2),I_i})\right]$$
\end{theorem}
Now, $P_{(1), k}$ is a Bernoulli with parameter $\frac{1}{2} + \varepsilon$, and $P_{(1), k}$ is a Bernoulli with parameter $\frac{1}{2} - \varepsilon$, so $D(P_{(1), k}\Vert P_{(1), k}) \le 16\varepsilon^2$, and so we have 
\begin{align*}
\sum_{k=1}^d \E_{(1)}\left[\sum_{i=1}^N \mathbbm{1}\{I_i = k\}\cdot D(P_{(1), I_i}||P_{(2),I_i})\right] &\le  \sum_{k=1}^d \E_{(1)}\left[\sum_{i=1}^N \mathbbm{1}\{I_i = k\}\cdot 16\varepsilon^2\right] \\
&=  16\varepsilon^2\cdot  \E_{(1)}\left[\sum_{k=1}^d\sum_{i=1}^N \mathbbm{1}\{I_i = k\}\right]  =16\varepsilon^2 m
\end{align*}
% In particular:
% \begin{align*}
% \log \left(\frac{\mathbb{P}_{(1)}(I_1, Y_1, I_2, Y_2,\cdots, Y_{j-1})}{\mathbb{P}_{(2)}(I_1, Y_1, I_2, Y_2, \cdots, Y_{j-1})} \right)
% &= \log \left(\frac{\prod_{j=1}^m\pi(I_j|I_1, Y_1, \cdots, Y_{j-1})P_{(1), I_j}(Y_j)}{\prod_{j=1}^m\pi(I_j|I_1, Y_1, \cdots, Y_{j-1})P_{(2), I_j}(Y_j)}\right )\\
% &=\log \left(\frac{\prod_{j=1}^m P_{(1), I_j}(Y_j)}{\prod_{j=1}^mP_{(2), I_j}(Y_j)}\right)
% =\sum_{j=1}^m\log \left(\frac{ P_{(1), I_j}(Y_j)}{P_{(2), I_j}(Y_j)}\right)
% \end{align*}
% \begin{align*}
%     D(\mathbb{P}_{(1)}|| \mathbb{P}_{(2)})
%     &=\E_{(1)}\left[\log \frac{d\mathbb{P}_{(1)}}{d\mathbb{P}_{(2)}}\right]
%     =\E_{(1)}\left[\sum_{j=1}^m \log \frac{P_{(1), I_t}}{P_{(2), I_t}}\right]\\
%     &=\sum_{j=1}^m \E_{(1)}\left[\log \frac{P_{(1), I_t}}{P_{(2), I_t}}\right]
%     =\sum_{j=1}^m \E_{(1)}\left[\E_{(1)}\left[\log \frac{P_{(1), I_t}}{P_{(2), I_t}}|I_t\right]\right]\\
%     &=\sum_{j=1}^m \E_{(1)}\left[D(P_{(1), I_t}||P_{(2), I_t})\right]
%     =\sum_{k = 1}^d \E_{(1)}\left[ \sum_{j=1}^m \mathbbm{1}\{I_t = k\} D(P_{(1), I_t}||P_{(2), I_t})\right]\\
%     &=\sum_{k = 1}^d D(P_{(1), k}||P_{(2), k})\E_{(1)}\left[\sum_{j=1}^m \mathbbm{1}\{I_t = k\} \right]
%     \le \sum_{k=1}^d \E_{(1)}[T_i]D(P_{(1), i}||P_{(2), i})\\
%     &= 8\varepsilon^2m
% \end{align*}
Putting this together, we have $\delta \ge e^{-16\varepsilon^2m}\implies m\ge \frac{1}{16\varepsilon^2}\log\frac{1}{\delta}$, and the result follows by replacing $\varepsilon$ with $\frac{1}{2}\varepsilon$.
\end{proof}
\begin{theorem}\label{lower_bound_3}
For any $d\ge 2$, there exist sets $\mathcal{X}\in \mathbb{R}^d, \mathcal{Y}\in \mathbb{R}$ of inputs and labels, and a distribution $P$ on $\mathcal{X}\times \mathcal{Y}$ such that any algorithm which solves Problem \ref{distribution_problem}, with $\varepsilon = 1$, requires at least $m=\frac{d}{3}\log \frac{1}{8\delta}$ samples.
\end{theorem}
\begin{proof}
All logarithms are base $4$. Consider instances in which $\mathcal{X} = \{e_1, e_2,\cdots, e_d\}$ where $e_i$ denotes the $i$th standard basis vector and the distribution over $\mathcal{X}$ is uniform. We take $Y = ZX^\top \beta^*$ for some $\beta^*$, where $Z$ is an independent Bernoulli random variable which is $1$ with probability $\frac{3}{4}$, and $0$ otherwise. Consider $d$ instances labelled with subscripts $(1), (2),\cdots, (d)$, one in which each of the $d$ standard basis is $\beta^*$, that is, $\beta^*_{(i)} = e_i$. Denote by $\beta_{j}$ the $j$th coordinate of $\beta$. For each instance, we have
\begin{restatable}{claim}{lossthree}\label{claim:loss3}
For all $i\in [d], \beta\in \mathbb{R}^d$, we have $\ell_{\beta^*_{(i)}}(\beta) \ge \frac{1}{4d}$ with equality when $\beta=\beta^*_{(i)}$
\end{restatable}
We would like our algorithm to return an estimate $\wh{\beta}$ which satisfies $\ell_{\beta^*}(\wh{\beta})<\frac{1}{2d}$. 
We first note that any choice of $\beta$ only succeeds to be this close to the optimal on a single instance. 
\begin{restatable}{claim}{lossfour}\label{claim:loss4}
Any $\beta \in \mathbb{R}^d$ can only satisfy $\ell_{\beta^*_{(i)}}(\wh{\beta})<\frac{1}{2d}$ for one $i \in [d]$. 
\end{restatable}
So, we may as well enforce that the algorithm return one of $e_1, e_2, \cdots, e_d$, since any other output can be mapped to one of these to improve the performance of the algorithm.

We will allow our algorithm to sample $N = \frac{d}{3}\log \frac{1}{\delta}$ rows total. Let $\mathcal{E}$ be the event that $Y_1, Y_2, \dots Y_N$ are all zero. Given any algorithm $\mathcal{A}$, let $F_\mathcal{A}$ denote the set of rows it samples fewer than $\log \frac{1}{\delta}$ times with probability at least $\frac{1}{2}$, in event $\mathcal{E}$. Because the total number of rows sampled is $\frac{d}{3}\log \frac{1}{\delta}$, there must be at least $\frac{2d}{3}$ rows which are sampled fewer than $\frac{1}{2}\log \frac{1}{\delta}$ times in expectation. 

By Markov's inequality, these rows are sampled fewer than $\log \frac{1}{\delta}$ times with probability at least $\frac{1}{2}$, and are thus all in $F_\mathcal{A}$. Let $B_\mathcal{A}$ denote the distribution over outputs $\wh{\beta}$ of $\mathcal{A}$ in event $\mathcal{E}$. Let $i_\mathcal{A}=\argmin_{j\in F_\mathcal{A}} B_\mathcal{A}(j)$. Denote by $G_\mathcal{A}$ the event that row $i_\mathcal{A}$ is sampled fewer than $\log\frac{1}{\delta}$ times; by construction we have $\mathbb{P}(G_\mathcal{A}) > \frac{1}{2}$. 

The subscripts are explicit because $F_\mathcal{A}, B_\mathcal{A}, i_\mathcal{A}, \mathbb{P}[G_\mathcal{A}]$ are properties of the algorithm and are independent of the instance with which it interacts. Consider the performance of this algorithm against the instance $\beta^*_{(i_\mathcal{A})}$. 

Let $Y_{(i_\mathcal{A}), j, k}$ denote the label returned to the algorithm when it queries $e_j$ for the $k$th time. Let $T_{(i_\mathcal{A})}=\min\{t| Y_{(i_\mathcal{A}), i_\mathcal{A}, t} = 1\}$. Denote by $E_{(i_\mathcal{A})}$ the event that $T_{(i_\mathcal{A})} > \log \frac{1}{\delta}$. Because $T_{(i_\mathcal{A})}$ is a geometric random variable, we have $\mathbb{P}[E_{(i_\mathcal{A})}]>\delta$. 

Now condition on the event $G_\mathcal{A}\cap E_{i_\mathcal{A}}$, which is an event with probability $\frac{1}{2}\delta$. Here our algorithm samples $i_\mathcal{A}$ fewer than $T_{i_\mathcal{A}}$ times, so it never sees a $1$ and its output distribution is $B_\mathcal{A}$. It returns $i\in F_\mathcal{A}\setminus \{i_\mathcal{A}\}$ with probability at least $1-B_\mathcal{A}(i_\mathcal{A}) \ge 1-\frac{1}{|F_\mathcal{A}|} \ge 1-\frac{3}{2d}\ge \frac{1}{4}$. In summary, even after $\frac{d}{3}\log\frac{1}{\delta}$ queries, no algorithm can return $\wh{\beta}$ with $\Vert X\wh\beta-y\Vert < (1+\varepsilon)\Vert X\beta^*-y\Vert$ with probability greater than $\frac{1}{8}\delta$. The result follows by replacing $\delta$ by $8\delta$.
\end{proof}
\begin{corollary}
Any algorithm that solves Problem \ref{matrix_problem} takes at least $\Omega(d\log\frac{1}{\delta}+\frac{d}{\varepsilon^2}+\frac{1}{\varepsilon^2}\log\frac{1}{\delta})$ samples for some $n=O(\frac{d\log\frac{d}{\delta}}{\varepsilon})$.
\end{corollary}
\begin{proof}
Each of the instances that demonstrate the lower bounds above, in Lemmas \ref{lower_bound_1}, \ref{lower_bound_2}, and \ref{lower_bound_3}, take $|\mathcal{X}| = d$, the results follows from Lemma \ref{distribution_to_matrix}.
\end{proof}

\subsection{Proof of Claims \ref{claim:hoeffding}, \ref{claim:loss1}, \ref{claim:loss2}, \ref{claim:loss3}, and \ref{claim:loss4}}
\hoeffding*
\begin{proof}[Proof of Claim \ref{claim:hoeffding}]
By assumption, we know that $X^\top \beta, Y \in [-1, 1]$, so, $|X^\top \beta - Y| \in [0, 2]$. So, for fixed $\beta$, by Hoeffding's on the rows of $\mathbf{X}\beta - y$, we have that if $n \ge \frac{8}{\varepsilon^2}\log \frac{2}{\delta'}$, then with probability at least $1 - \delta'$, 
\begin{align}\left(1-\frac{\varepsilon}{2}\right)\E_{(X, Y)\sim P}\left[|X^\top \beta - Y|\right] \le \frac{1}{n}\Vert \mathbf{X}\beta-y\Vert_1 \le \left(1+\frac{\varepsilon}{2}\right)\E_{(X, Y)\sim P}\left[|X^\top \beta - Y|\right]\label{eq:hoef}\end{align}
Now, we construct a $\frac{\varepsilon}{2d}$-covering $S$ of the unit $\ell_\infty$ ball $H$, with fewer than $\left(\frac{4d}{\varepsilon}\right)^d$ elements, so that for any $\beta$, there is some $\beta_c \in S$ such that $\Vert \beta - \beta_c\Vert_\infty \le \frac{\varepsilon}{2d}$. To do this, simply take $S = \{\beta : \beta_i = k\frac{\varepsilon}{2d}, k \in \mathbb{Z} \cap [-2d/\varepsilon, 2d/\varepsilon]\}$. 

Note that $\mathbf{X}$ has rows on the hypercube. So, if we denote $x_{i, j}$ to be the entry of $\mathbf{X}$ in the $i$th row and $j$th column, then $x_{i, j} \in \{-1, 1\}$. Therefore, for any $\beta$,  
\begin{align*}
    \Vert \mathbf{X}\beta\Vert_1 &= \sum_{i = 1}^n |x_i^\top \beta| \le \sum_{i = 1}^n \sum_{j = 1}^d |x_{i, j} \beta_j| \le \sum_{i = 1}^n \sum_{j = 1}^d |\beta_j|\le nd\Vert \beta\Vert_\infty
\end{align*}
Therefore, we can apply Hoeffding's, as in (\ref{eq:hoef}), with $\delta' = \delta \left(\frac{\varepsilon}{4d}\right)^d$, and union bound over the set $S$, to get that for any $\beta \in S$, with probability at least $1 - \delta$, (\ref{eq:hoef}) holds. 

Then, for any $\beta \in H$, by the covering property, we can find some $\beta_c \in S$ such that \begin{align}\Vert \beta - \beta_c\Vert_\infty \le \frac{\varepsilon}{d}\implies \Vert \mathbf{X}\beta - \mathbf{X}\beta_c\Vert_1 \le n\varepsilon.\label{eq:norm}\end{align} We have
\begin{align*}
    \Vert \mathbf{X} \beta_c - y\Vert_1 - \Vert \mathbf{X}\beta_c - \mathbf{X}\beta\Vert_1 \le \Vert \mathbf{X}\beta - y\Vert_1 \le \Vert \mathbf{X} \beta - \mathbf{X}\beta_c\Vert_1 + \Vert \mathbf{X}\beta_c - y\Vert_1
\end{align*}
So, combining (\ref{eq:hoef}) and (\ref{eq:norm}), and dividing by $n$, we finally have that if $n \ge \frac{8}{\varepsilon^2}\left(\log \frac{2}{\delta} + d\log \frac{4d}{\varepsilon}\right)$, then for all $\beta \in H$, 
\begin{align*}
    (1-\varepsilon)\E_{(X, Y)\sim P}\left[|X^\top \beta - Y|\right] \le \frac{1}{n}\Vert \mathbf{X}\beta-y\Vert_1 \le (1+\varepsilon)\E_{(X, Y)\sim P}\left[|X^\top \beta - Y|\right]
\end{align*}
\end{proof}
\lossone*
\begin{proof}[Proof of Claim \ref{claim:loss1}]
The $\ell_1$ error for the correct $\beta$ is given by
\begin{align*}
&\E_{(X, Y)\sim P} \big|X^\top\beta^*-Y\big| \\
& \quad = \E_{X}[E_{Y\sim P(\cdot|X)} \big||X^\top\beta^*-Y|] && \text {by independence}\\
& \quad = \E_{X}[(\frac{1}{2}+\varepsilon)\big|X^\top\beta^*-X^\top\beta^*|+(\frac{1}{2}-\varepsilon)\big|X^\top\beta^*+X^\top\beta^*\big|]\\
& \quad = \E_{X}[(1-2\varepsilon)\big|X^\top\beta^*|] && \beta^* \in \mathcal{H}\\
& \quad = 1-2\varepsilon
\end{align*}
\end{proof}
\losstwo*
\begin{proof}[Proof of Claim \ref{claim:loss2}]
\begin{align*}
&\E_{(X, Y)\sim P} \big|X^\top\beta-Y|\big| \\
& \quad = \E_{X}\left[E_{Y\sim P(\cdot|X)} \big|X^\top\beta-Y|\big|\right]\\
& \quad = \E_{X}\left[\left(\frac{1}{2}+\varepsilon\right)\big|X^\top\beta-X^\top\beta^*\big|+\left(\frac{1}{2}-\varepsilon\right)|X^\top\beta+X^\top\beta^*\big|\right]\\
& \quad = (1-2\varepsilon)+2\varepsilon\E_X[X^\top\beta-X^\top\beta^*]\\
& \quad = (1-2\varepsilon)+2\varepsilon\frac{1}{d}||\beta-\beta^*||_1
\end{align*}
\end{proof}
\losstwofive*
\begin{proof}[Proof of Claim \ref{claim:loss2.5}]
\begin{align*}
l(\beta)+l(\beta)
&= 2-4\varepsilon+\frac{2\varepsilon}{d}\Vert \beta^*_{(1)}-\beta\Vert_1+\frac{2\varepsilon}{d}\Vert \beta^*_{(2)}-\beta\Vert_1\\
&\ge 2-4\varepsilon+\frac{2\varepsilon}{d}\Vert \beta^*_{(2)}-\beta^*_{(1)}\Vert_1\\
&= 2
\end{align*}
$$\implies\max \{\ell_{\beta^*_{(1)}}(\beta)-\ell_{\beta^*_{(1)}}(\beta^*_{(1)}), \ell_{\beta^*_{(2)}}(\beta)-\ell_{\beta^*_{(2)}}(\beta^*_{(2)})\}>2\varepsilon, \hfill \forall\beta\in\mathbb{R}^d$$
\end{proof}
\lossthree*
\begin{proof}[Proof of Claim \ref{claim:loss3}]
\begin{align*}
\ell_{\beta^*_{(i)}}(\beta) 
&= \frac{1}{d}\sum_{j\ne i}|\beta_{j}|+\frac{\frac{1}{2}+\varepsilon}{d}|1-\beta_{i}|+\frac{\frac{1}{2}-\varepsilon}{d}|\beta_{i}|\\
&\ge \frac{\frac{1}{2}-\varepsilon}{d}\left(|\beta_i|+|1-\beta_i|\right)+\frac{2\varepsilon}{d}|1-\beta_{i}|\ge \frac{\frac{1}{2}-\varepsilon}{d}
\end{align*}
\end{proof}
\lossfour*
\begin{proof}[Proof of Claim \ref{claim:loss4}]
Indeed, suppose $\beta$ was such that $\ell_{\beta^*_{(I)}}(\beta), \ell_{\beta^*_{(J)}}(\beta)<\frac{1}{2d}$. Then we must have \begin{align*}
\frac{1}{2d}
&\ge \ell_{\beta^*_{(I)}}(\beta) \\
&= \frac{1}{d}\sum_{j\ne I}|\beta_{j}|+\frac{\frac{1}{2}-\varepsilon}{d}\left(|\beta_I|+|1-\beta_i|\right)+\frac{2\varepsilon}{d}|1-\beta_{I}|\\
&\ge \frac{1}{d}\sum_{j\ne I}|\beta_{j}|+\frac{\frac{1}{2}-\varepsilon}{d}+\frac{2\varepsilon}{d}|1-\beta_{I}|\\
\iff \varepsilon 
&\ge \sum_{j\ne I}|\beta_{j}|+2\varepsilon|1-\beta_{I}|\\
&\ge \sum_{j\ne I}|\beta_{j}|+2\varepsilon-2\varepsilon|\beta_{I}|\\
\iff 2|\beta_I|&\ge \Vert \beta\Vert_1+2\varepsilon
\end{align*}
Similarly for $J$, so we would have $\Vert \beta\Vert \ge |\beta_I|+|\beta_J|\ge \Vert \beta\Vert_1+2\varepsilon$. 
\end{proof}

\end{document}